\documentclass{article}

\usepackage{microtype}
\usepackage{graphicx}
\usepackage{subfigure}
\usepackage{booktabs} 

\usepackage{amssymb}
\usepackage{pifont}
\newcommand{\cmark}{\ding{51}}%
\newcommand{\xmark}{\ding{55}}%
\newcommand{\Cat}{\operatorname{Cat}}

\newcommand{\Dir}{\operatorname{Dir}}
\newcommand{\Lip}{\operatorname{Lip}}
\newcommand{\ReLU}{\operatorname{ReLU}}

\newcommand{\drop}{\operatorname{drop}}
\newcommand{\epochs}{\operatorname{epochs}}
\newcommand{\prior}{\operatorname{prior}}
\newcommand{\data}{\operatorname{data}}
\newcommand{\post}{\operatorname{post}}
\newcommand{\id}{\operatorname{id}}
\newcommand{\ood}{\operatorname{ood}}
\newcommand{\tr}{\operatorname{tr}}
\newcommand{\Clip}{\operatorname{Clip}}
\newcommand{\Softplus}{\operatorname{Softplus}}

\usepackage{hyperref}

\usepackage[accepted]{icml2024}

\usepackage{amsmath}
\usepackage{bbm}
\usepackage{amssymb}
\usepackage{mathtools}
\usepackage{amsthm}

\usepackage[capitalize,noabbrev]{cleveref}

\theoremstyle{plain}
\newtheorem{theorem}{Theorem}[section]

\newtheorem{lemma}[theorem]{Lemma}
\newtheorem{corollary}[theorem]{Corollary}
\theoremstyle{definition}
\newtheorem{definition}[theorem]{Definition}
\newtheorem{assumption}[theorem]{Assumption}
\theoremstyle{remark}

\usepackage[textsize=tiny]{todonotes}

\icmltitlerunning{Uncertainty Estimation by Density Aware Evidential Deep Learning}

\begin{document}

\twocolumn[
\icmltitle{Uncertainty Estimation by Density Aware Evidential Deep Learning}



\icmlsetsymbol{equal}{*}

\begin{icmlauthorlist}
\icmlauthor{Taeseong Yoon}{yyy}
\icmlauthor{Heeyoung Kim}{yyy}
\end{icmlauthorlist}
\icmlaffiliation{yyy}{Department of Industrial and Systems Engineering, KAIST, Daejeon, Republic of Korea}
\icmlcorrespondingauthor{Heeyoung Kim}{heeyoungkim@kaist.ac.kr}

\icmlkeywords{Machine Learning, ICML}

\vskip 0.3in
]



\printAffiliationsAndNotice{}  
\begin{abstract}
Evidential deep learning (EDL) has shown remarkable success in uncertainty estimation. However, there is still room for improvement, particularly in out-of-distribution (OOD) detection and classification tasks. The limited OOD detection performance of EDL arises from its inability to reflect the distance between the testing example and training data when quantifying uncertainty, while its limited classification performance stems from its parameterization of the concentration parameters. To address these limitations, we propose a novel method called \textit{Density Aware Evidential Deep Learning (DAEDL)}. DAEDL integrates the feature space density of the testing example with the output of EDL during the prediction stage, while using a novel parameterization that resolves the issues in the conventional parameterization. We prove that DAEDL enjoys a number of favorable theoretical properties. DAEDL demonstrates state-of-the-art performance across diverse downstream tasks related to uncertainty estimation and classification. 
\end{abstract}

\section{Introduction}

\label{Introduction}
In recent years, artificial intelligence (AI) has achieved remarkable advancements, demonstrating state-of-the-art performance across various domains. Despite these significant strides, applying AI models to real-world problems poses challenges. Relying solely on AI models for critical decision-making is considered risky. To safely deploy AI models in high-risk domains such as healthcare, finance, and manufacturing, they must possess the capability to represent the uncertainty of their outcomes accurately. However, it is widely acknowledged that modern neural networks often lack proper calibration, and struggle to precisely represent the uncertainty associated with their predictions \cite{guo2017calibration}. Various methods, including deep ensemble \cite{lakshminarayanan2017simple}, Monte Carlo dropout \cite{gal2016dropout}, and Bayesian neural networks \cite{blundell2015weight}, have been proposed to quantify the uncertainty of AI models. These methods exhibit reasonable performance underpinned by a solid theoretical foundation. Nevertheless, their practical applicability in real-world settings is hindered by the necessity of multiple forward passes, making them less feasible. To address this limitation, researchers have explored models with the capability to quantify uncertainty in a single forward pass, aiming to enhance the practical applicability of uncertainty estimation models for real-world problems.

Dirichlet-based uncertainty (DBU) models \cite{malinin2018predictive, sensoy2018evidential,malinin2019reverse, charpentier2020posterior, charpentier2021natural, ulmer2023prior} have emerged as a promising avenue among models capable of quantifying uncertainty in a single forward pass. Unlike conventional classification models that directly predict class probabilities, DBU models adopt a distinctive approach by predicting the distribution of class probabilities. 
Evidential deep learning (EDL) \cite{sensoy2018evidential}, which employs a Dirichlet distribution to simultaneously quantify belief mass for each class and uncertainty mass, stands out as a prominent example of the DBU models. EDL is distinguished for its simplicity in implementation and impressive uncertainty estimation performance in various tasks, especially in out-of-distribution (OOD) detection \cite{sensoy2018evidential, deng2023uncertainty}. 

Despite EDL's notable success, there is still room for improvement. First, EDL may fail to accurately reflect the distance between the testing examples and training data when quantifying predictive uncertainty, potentially resulting in a decline in OOD detection performance. 
To shed light on our hypothesis, we conducted a toy experiment using a two moons dataset \cite{liu2020simple}. In \cref{Fig1}, uncertainty representations (predictive variance) obtained by Softmax and EDL are depicted in (a) and (b) along with the training data of two classes (blue and orange) and OOD data (red). 
The ideal uncertainty estimation model should yield low predictive uncertainty for testing examples that are near the training data, with the uncertainty increasing as they move farther away from the training data. 
However, as illustrated in \cref{Fig1}, Softmax and EDL exhibit high uncertainty (yellow) primarily along the decision boundary and low uncertainty (purple) elsewhere. Notably, Softmax and EDL assign low uncertainty even to OOD data. 
In contrast, our proposed method in (c), which will be detailed below, yielded the desired uncertainty estimation results. 
Second, EDL exhibits limited classification performance, restricting its suitability for real-world problems, where classification typically takes precedence over uncertainty estimation. 
We hypothesize that this limitation arises from the conventional parameterization of EDL, specifically the challenge of estimating an appropriate magnitude of the {\it evidence} in relation to the concentration parameters of the Dirichlet distribution,  due to the absence of an explicit range for the magnitude of the evidence. 

To address these limitations, we propose a novel method called \textit{Density Aware Evidential Deep Learning (DAEDL)}. 
First, to enable an uncertainty estimate that reflects the distance between the testing examples and training data, DAEDL integrates the feature space density of the testing example with the output of EDL during the prediction stage. For the density estimation, DAEDL employs Gaussian discriminant analysis (GDA) in the feature space, inspired by \citet{mukhoti2023deep}, which allows it to estimate the density without additional training. Notably, the feature space encapsulates relevant features essential for both uncertainty estimation and classification, whereas direct density estimation in the input space is computationally demanding and susceptible to the curse of dimensionality \cite{choi2018generative, nalisnick2019detecting}. 
Second, to overcome the potential limitation of the conventional parameterization of EDL, DAEDL introduces an alternative novel parameterization that resolves the issues arising from the lack of an explicit range for the evidence. Additionally, DAEDL adopts an exponential activation function, in contrast to ReLU in EDL, which allows it to establish a connection with the softmax model, thereby further enhancing classification accuracy. 

\begin{figure}[htbp]
\centerline{\includegraphics[width=\columnwidth]{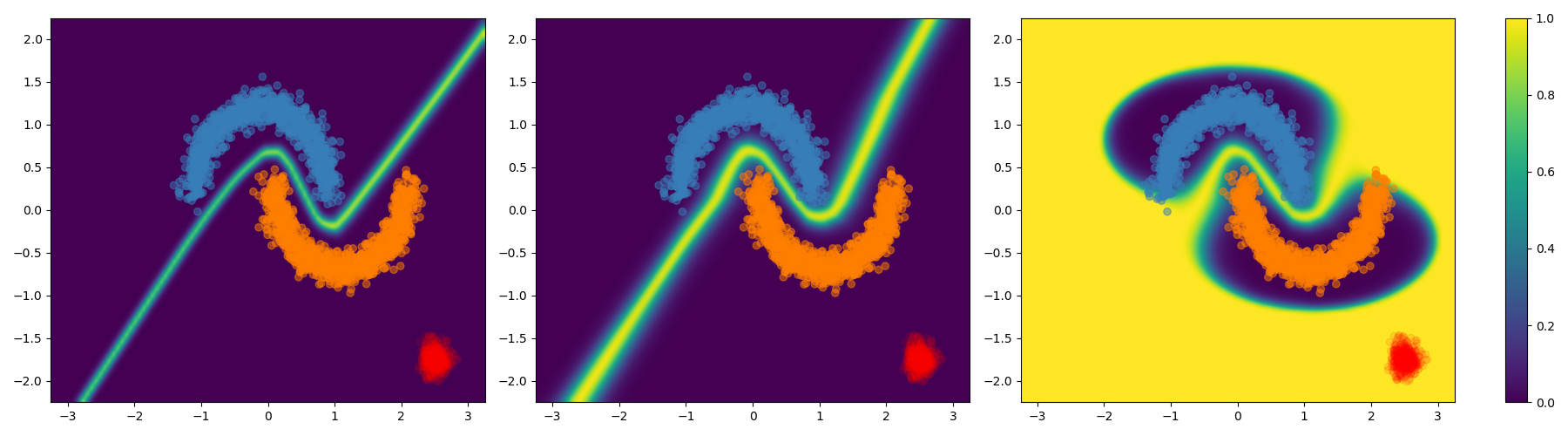}}
\begin{small}
\quad \quad \textbf{(a) Softmax} \quad \qquad \textbf{(b) EDL} \qquad \quad \ \ \textbf{(c) DAEDL (ours)}
\end{small}
\caption{Uncertainty representations on the two moons dataset. (a) Softmax, (b) EDL, and (c) DAEDL (ours).}
\label{Fig1}
\end{figure}

We establish that DAEDL exhibits favorable theoretical properties, elucidating the reasons for its superior uncertainty estimation performance over conventional EDL. First, we prove that DAEDL generates a uniform predictive distribution over classes for OOD data. Second, we prove that the predictive distribution of DAEDL can be conceptualized as an \textit{adaptive temperature scaled} softmax model, which has demonstrated effective performance in improving model calibration and OOD detection \cite{balanya2022adaptive, joy2023sample, krumpl2024ats}. Third, we prove that DAEDL can be interpreted as predicting an \textit{input-dependent posterior distribution of the Dirichlet-Categorical model} \cite{charpentier2020posterior} with an improper prior $\boldsymbol{\pi} \sim \Dir(\mathbf{0})$, while typical DBU models \cite{sensoy2018evidential, charpentier2020posterior, charpentier2021natural, deng2023uncertainty} can be seen as utilizing a uniform prior $\boldsymbol{\pi} \sim \Dir (\mathbf{1})$ under this framework. This interpretation implies that DAEDL employs an improper prior in a Bayesian context, thereby addressing the challenge of specifying an appropriate prior distribution. Fourth, we prove that DAEDL's predictive uncertainty of a testing example is proportional to its distance from the training data manifold, under mild assumptions. This property, formally defined as \textit{distance awareness} \cite{liu2020simple}, has been demonstrated to enhance the quality of uncertainty estimation. 

DAEDL consistently demonstrates state-of-the-art performance across various downstream tasks related to uncertainty estimation, including OOD detection, confidence calibration, and distribution shift detection, as well as achieving superior performance in image classification.

\section{Evidential Deep Learning}
\label{Evidential Deep Learning}
EDL \cite{sensoy2018evidential} stands as one of the pioneering works in the class of DBU models. The development of EDL is grounded in the principles of subjective logic (SL) \cite{josang1997artificial, josang2016subjective} and Dempster-Shafer Theory of Evidence (DST) \cite{dempster1968generalization, shafer1976mathematical}. In a classification problem with $C$ classes, DST assigns a belief mass $b_{c}, \ \forall c \in [C]$ for each class, which measures the evidence in favor of each class, and the uncertainty mass $u$, which captures the overall uncertainty. These values are all non-negative and subject to the constraint $u + \sum_{c=1}^{C} b_{c} = 1$. SL models the belief assignment framework of DST using a Dirichlet distribution. The concentration parameters of the Dirichlet distribution, denoted as $\boldsymbol{\alpha} = [\alpha_{1}, \alpha_{2}, \cdots \alpha_{C}]$, is parameterized as $\alpha_{c} = 1 + e_{c}$, $\forall c \in [C]$, where $e_{c}$ denotes the evidence for the $c$th class.
The belief and uncertainty values are computed as $b_{c} = e_{c}/\alpha_{0}$ and $u = C/\alpha_{0}$, where $\alpha_{0} = \sum_{c=1}^{C} \alpha_{c}$ denotes the precision of the Dirichlet distribution. A higher $\alpha_{0}$ corresponds to a sharper and more confident distribution.

EDL performs classification by estimating the evidence vector $\boldsymbol{e} = [e_{1}, e_{2}, \cdots, e_{C}], \forall e_{c} > 0$, using a neural network. Specifically, the evidence vector of input $\mathbf{x} \in \mathbb{R}^{D}$ is computed as $\mathbf{e}_{\boldsymbol{\theta}, \boldsymbol{\phi}}(\mathbf{x}) = h(g_{\boldsymbol{\phi}}(f_{\boldsymbol{\theta}}(\mathbf{x})))$, where $f_{\boldsymbol{\theta}} : \mathbb{R}^{D} \rightarrow \mathbb{R}^{H}$, $g_{\boldsymbol{\phi}}: \mathbb{R}^{H} \rightarrow \mathbb{R}^{C}$, and $h:\mathbb{R}^{C} \rightarrow \mathbb{R}_{+}^{C}$ is the feature extractor, classifier, and activation function, respectively. Here, $\boldsymbol{\theta}$ and  $\boldsymbol{\phi}$ represent the parameters of the feature extractor and classifier, respectively, while $D$, $H$, and $C$ denote the dimension of the input data, dimension of feature representations, and the number of classes, respectively. 
Following the parametrization used in SL, the concentration parameters of the Dirichlet distribution are obtained as 
\begin{equation}\label{conven}
\boldsymbol{\alpha}_{\boldsymbol{\theta}, \boldsymbol{\phi}}(\mathbf{x}) = \textbf{1} + \mathbf{e}_{\boldsymbol{\theta}, \boldsymbol{\phi}}(\mathbf{x}), 
\end{equation}
where $\mathbf{1}= [1,1,\cdots1] \in \mathbb{R}^{C}$ is the vector of ones.

EDL is optimized using a loss function that consists of two components: i) expected mean squared error (MSE), responsible for accurate uncertainty-aware classification, and ii) Kullback-Leibler (KL) divergence penalty, ensuring the desired uncertainty behavior of the concentration parameters. For sample $(\mathbf{x}_{i}, \mathbf{y}_{i})$, where $\mathbf{y}_{i}$ is the one-hot encoded label, the loss function is formulated as follows: 
\begin{align}\label{loss}
\mathcal{L}^{(i)}(\boldsymbol{\theta}, \boldsymbol{\phi}) &= \mathbb{E}_{\boldsymbol{\pi} \sim \Dir(\boldsymbol{\alpha}_{\boldsymbol{\theta},\boldsymbol{\phi}}(\mathbf{x}_{i}))} [\lVert \mathbf{y}_{i}-\boldsymbol{\pi} \rVert_{2}^{2}] \\ \nonumber
&+ \lambda D_{KL}[\Dir(\boldsymbol{\tilde{\alpha}_{\boldsymbol{\theta}, \boldsymbol{\phi}}}(\mathbf{x}_{i})||\Dir(\mathbf{1})],  
\end{align}
where $\tilde{\boldsymbol{\alpha}}_{\boldsymbol{\theta}, \boldsymbol{\phi}}(\mathbf{x}_{i}) = \boldsymbol{\alpha}_{\boldsymbol{\theta}, \boldsymbol{\phi}}(\mathbf{x}_{i}) \odot (\mathbf{1}-\mathbf{y}_{i}) + \mathbf{y}_{i}$, and $\lambda$ is a regularization parameter.

\section{Density Aware Evidential Deep Learning}
\label{Density Aware Evidential Deep Learning}
\subsection{Model Overview}
\label{Model Overview}
DAEDL closely follows the conventional structure of EDL outlined in \cref{Evidential Deep Learning}, but introduces two significant modifications designed to address the limitations of EDL: 
(i) the adoption of an alternative parameterization (\cref{Contribution1 : param}), and (ii) the integration of feature space density (\cref{Contribution2 : density}). DAEDL is trained using the loss specified in Eq.\eqref{loss}. After training, we obtain the feature extractor $f_{\hat{\boldsymbol{\theta}}}$ and classifier $g_{\hat{\boldsymbol{\phi}}}$ with the optimized parameters $\hat{\boldsymbol{\theta}}$ and $\hat{\boldsymbol{\phi}}$. The training procedure for DAEDL is presented in Algorithm 1 in \cref{Appendix : Algorithm}.

In contrast to conventional EDL, DAEDL employs \textit{spectral normalization} \cite{miyato2018spectral} in the feature extractor $f_{\boldsymbol{\theta}}$, facilitating a meaningful density estimate even when a simple density estimator, such as GDA, is used. Spectral normalization has been widely utilized for single forward pass uncertainty estimation models to achieve a regularized feature space  \cite{liu2020simple, van2021feature, mukhoti2023deep}. Notably, we are the first to adopt it for DBU models. By employing spectral normalization, we bound the distance in the feature space by the distance in the input space, preventing feature representations from becoming overly sensitive to meaningless perturbations in the input space \cite{liu2020simple}. More detailed descriptions of spectral normalization are provided in \cref{Appendix : Spectral Normalization}.

\subsection{Alternative Parameterization}
\label{Contribution1 : param}
We propose a novel parameterization for the concentration parameters of DAEDL, which overcomes the limitations of the conventional scheme in EDL while establishing a connection with the softmax model. The conventional parameterization for the concentration parameters of EDL is expressed as in Eq.\eqref{conven}. 
This parameterization has inherent limitations due to the challenge of achieving an appropriate balance between $\mathbf{1}$ and $\mathbf{e}_{\boldsymbol{\theta}, \boldsymbol{\phi}}(\mathbf{x})$. As there is no explicit range for $\mathbf{e}_{\boldsymbol{\theta}, \boldsymbol{\phi}}(\mathbf{x})$, $\mathbf{1}$ may dominate $\mathbf{e}_{\boldsymbol{\theta}, \boldsymbol{\phi}}(\mathbf{x})$, potentially leading to counter-intuitive outcomes for the expected class probabilities. For example, in the case with $C=3$, given a highly likely ID data point $\mathbf{x}_{\id}$ and a highly-peaked evidence vector computed as $\mathbf{e}_{\boldsymbol{\theta}, \boldsymbol{\phi}}(\mathbf{x}_{\id}) = [1,0,0]$, the expected class probability is derived as $\bar{\boldsymbol{\pi}}_{\id} = [0.5,0.25,0.25]$, contradicting the intuition that the class probability for a highly likely ID data point should be more strongly peaked toward the corresponding class. We hypothesize that these counter-intuitive results may hinder the model from learning the decision boundary accurately, leading to degraded classification performance. 

To address the above limitations and enhance the classification performance of EDL, DAEDL introduces two concurrent modifications to the conventional parameterization: i) the removal of $\mathbf{1}$ in Eq.\eqref{conven} and ii) the adoption of the exponential activation function, instead of ReLU in EDL. We eliminate $\mathbf{1}$ to address the challenge associated with balancing the magnitudes between  $\mathbf{1}$ and $\mathbf{e}_{\boldsymbol{\theta}, \boldsymbol{\phi}}(\mathbf{x})$. We argue that $\mathbf{1}$ is not necessarily an essential component in the model, hindering the effectiveness of EDL in the classification task. By removing it, DAEDL allows the model to learn a Dirichlet distribution solely from the data. Subsequently, we adopt an exponential function as the activation function to establish a close connection with the softmax model, replacing ReLU \cite{sensoy2018evidential} or Softplus \cite{deng2023uncertainty}. Then, the concentration parameters of DAEDL during the training stage can be formulated as $\boldsymbol{\alpha}_{\boldsymbol{\theta}, \boldsymbol{\phi}}(\mathbf{x}) = \exp \big( g_{\boldsymbol{\phi}}(f_{\boldsymbol{\theta}}(\mathbf{x}) \big)$. As summarized in \cref{Parameter Modeling Option}, employing such parameterization aligns the expected class probability of DAEDL with the output of the softmax model. Given that the softmax model generally outperforms conventional EDL in terms of classification accuracy, we expect that DAEDL will improve upon EDL in classification performance. 

\begin{table}[htbp]
\caption{
Comparison of the concentration parameter ($\alpha_{c}$) and expected class probability ($\mathbb{E}_{\boldsymbol{\pi} \sim \Dir(\boldsymbol{\alpha})}[\pi_{c}]$) during the training stage between the standard EDL \cite{sensoy2018evidential, deng2023uncertainty} and DAEDL. $z_{c}(\mathbf{x}) = (g_{\boldsymbol{\phi}}(f_{\boldsymbol{\theta}}(\mathbf{x})))_{c}$ represents the logit for each class $\forall c \in [C]$, and $h$ is an activation function that ensures the non-negativity of the evidence vector (e.g., ReLU \cite{sensoy2018evidential} and Softplus \cite{deng2023uncertainty}).
}
\label{Parameter Modeling Option}
\vskip 0.15in
\begin{center}
\begin{small}
\begin{sc}
\begin{tabular}{@{}llll@{}}
\toprule
&\textbf{EDL} & \textbf{DAEDL} \\ \midrule
$\alpha_{c}$ & $1 + h(z_{c}(\mathbf{x}))$ & $ \exp(z_c(\mathbf{x}))$ \\ \midrule
$\mathbb{E}_{\boldsymbol{\pi} \sim \Dir(\boldsymbol{\alpha})}[\pi_{c}]$ & $\frac{1 + h(z_{c}(\mathbf{x}))}{C + \sum_{c=1}^{C} h(z_{c}(\mathbf{x}))}$ & $\frac{\exp(z_{c}(\mathbf{x}))}{\sum_{c=1}^{C} \exp(z_{c}(\mathbf{x}))}$ \\ \bottomrule
\end{tabular} 
\end{sc}
\end{small}
\end{center}
\vskip -0.15in
\end{table}

\subsection{Integration of Feature Space Density}
\label{Contribution2 : density}
 DAEDL leverages the feature space density of a testing example during prediction to obtain an uncertainty estimate that reflects its distance from the training data. 
Specifically, we employ GDA as the density estimator in the feature space, and integrate the estimated feature space density of the testing example with the logits in the prediction stage. 

\paragraph{Density estimation.}
DAEDL employs GDA as a density estimator in the feature space, inspired by its efficiency and the ability to operate without additional training \cite{mukhoti2023deep}. Given the training dataset $\mathcal{D}_{\tr} = \{(\mathbf{x}_{i}, y_{i})\}_{i=1}^{N}$ and the trained feature extractor $f_{\hat{\boldsymbol{\theta}}}$, the parameters of GDA for each class $\forall c \in [C]$ are obtained as follows: 
\begin{align*}
\hat{\omega}_{c} &= \frac{N_{c}}{N}, \quad 
\hat{\boldsymbol{\mu}}_{c} = \frac{1}{N_{c}} \sum_{\{i:y_{i} = c\}}^{} f_{\hat{\boldsymbol{\theta}}}(\mathbf{x}_{i}), \\
\hat{\Sigma}_{c} &= \frac{1}{N_{c}-1} \sum_{\{i:y_{i} = c\}}^{} (f_{\hat{\boldsymbol{\theta}}}(\mathbf{x}_{i}) - \hat{\boldsymbol{\mu}}_{c}) (f_{\hat{\boldsymbol{\theta}}}(\mathbf{x}_{i}) - \hat{\boldsymbol{\mu}}_{c})^{T}, 
\end{align*}
where $\hat{\omega}_{c}$, $\hat{\boldsymbol{\mu}}_{c}$, and $\hat{\Sigma}_{c}$ represent the weight, mean vector, and covariance matrix for each class, respectively. In addition, $N$ denotes the total number of training data points, and $N_{c} = \sum_{i=1}^{N} \mathbbm{1}_{\{y_{i} = c\}}$ represents the number of data points for each class $\forall c \in [C]$. The density estimation algorithm for DAEDL is provided in Algorithm 2 in \cref{Appendix : Algorithm}.

\paragraph{Prediction.} In the prediction stage, we estimate the feature space density of the testing example and integrate it with the output of EDL to obtain a reliable uncertainty estimate that reflects the distance between the testing example and the training data. The graphical representation of the prediction stage is depicted in \cref{Fig2}.

\begin{figure}[htbp]
\includegraphics[width = \columnwidth]{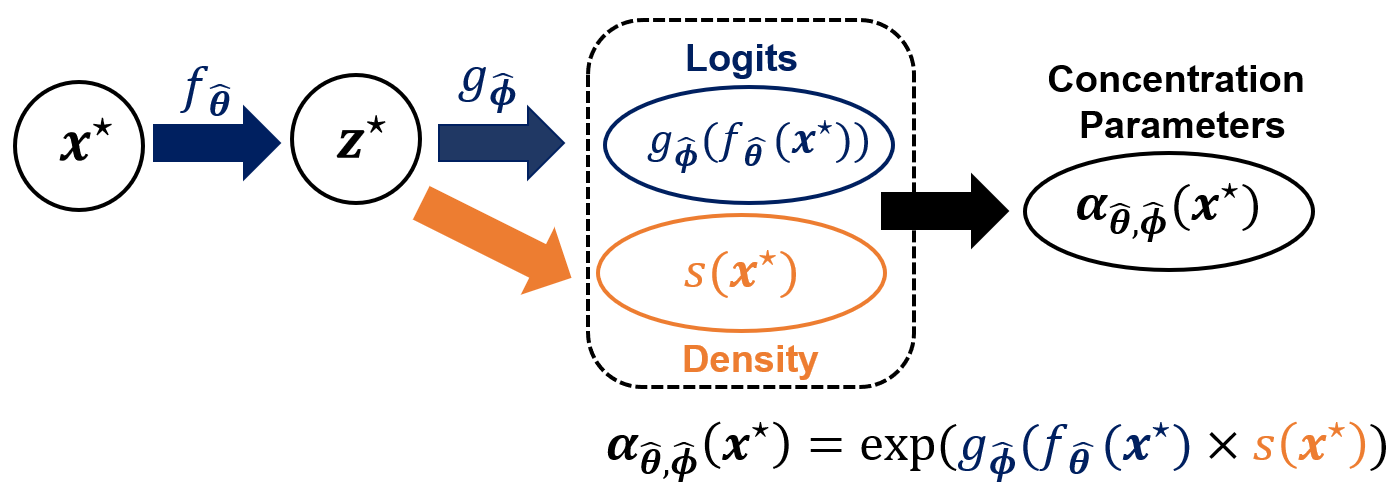}
\vspace{-0.2 in}
\caption{Graphical representation of the prediction stage of DAEDL. The process begins with the computation of the logits (blue),  followed by the estimation of the normalized feature space density of the testing example (orange). These two components are integrated to derive the concentration parameters (black).}
\label{Fig2}
\end{figure}

First, we estimate the feature space density of a testing example using a fitted GDA model on the training data. For testing example $\mathbf{x}^{\star} \in \mathbb{R}^{D}$, the feature space density is obtained as follows: 
\begin{align*}
p(\mathbf{z}^{\star} = f_{\hat{\boldsymbol{\theta}}}(\mathbf{x}^{\star}) ) = \sum_{c=1}^{C} \hat{\omega}_{c} \ \mathcal{N}(\mathbf{z}^{\star} = f_{\hat{\boldsymbol{\theta}}}(\mathbf{x}^{\star}) |\hat{\boldsymbol{\mu}}_{c}, \hat{\Sigma}_{c}),
\end{align*}
where $\mathbf{z}^{\star} \in \mathbb{R}^{H}$ is a feature representation of $\mathbf{x}^{\star}$. 
In practice, we first calculate the logarithm of the feature space density, and then normalize it to a range between $[0,1]$ to avoid challenges related to parameter divergence, leveraging the minimum and maximum values observed in the logarithm of the feature space density of the training data $\mathcal{X}_{\tr} = \{\mathbf{x}_{i}\}_{i=1}^{N}$ and the $\Clip$ function defined as $\Clip(x) = \max(0, \min(1,x))$. The normalized feature space density for $\mathbf{x}^{\star}$ is expressed as follows: 
\begin{align*}
s(\mathbf{x}^{\star}) = \Clip \left( \frac{\log p(\mathbf{z}^{\star} = f_{\hat{\boldsymbol{\theta}}}(\mathbf{x}^{\star}))-d_{\min}}{d_{\max} - d_{\min}} \right), 
\end{align*}
where $d_{\min} = \min_{\mathbf{x} \in \mathcal{X}_{tr}} \left \{ \log p( f_{\hat{\boldsymbol{\theta}}}(\mathbf{x})) \right \}$ and $d_{\max} = \max_{\mathbf{x} \in \mathcal{X}_{tr}} \left \{\log p(f_{\hat{\boldsymbol{\theta}}}(\mathbf{x})) \right \}$.

Subsequently, we combine the normalized feature space density (i.e., $ s(\mathbf{x}^{\star})$) with the logits (i.e., $g_{\hat{\boldsymbol{\phi}}}(f_{\hat{\boldsymbol{\theta}}}(\mathbf{x}^{\star})$) to obtain the concentration parameters of DAEDL. Specifically, we multiply the normalized feature space density with the logits, and apply the exponential activation function to derive the concentration parameters as follows:
\begin{align*}
\boldsymbol{\alpha}_{\hat{\boldsymbol{\theta}}, \hat{\boldsymbol{\phi}}}(\mathbf{x}^{\star})&= \exp \left(g_{\hat{\boldsymbol{\phi}}}(f_{\hat{\boldsymbol{\theta}}}(\mathbf{x}^{\star})) \times s(\mathbf{x}^{\star})\right).
\end{align*}
This process can be interpreted as scaling the logits with the estimated confidence level of the prediction before applying the activation function. 
This integration strategy demonstrates effective empirical performance, underpinned by favorable theoretical properties, including \cref{Relationship with AdaTS}, \cref{Feature Distance Awareness}, and \cref{Input Distance Awareness}. 
The algorithm for the prediction of DAEDL is provided in Algorithm 3 in \cref{Appendix : Algorithm}.

\section{Theoretical Analysis}
\label{Theoretical Analysis}
We establish the theoretical foundations of DAEDL. 
First, we prove that DAEDL generates a uniform predictive distribution over classes for highly likely OOD testing examples (\cref{Uniform for OOD Data}). This indicates that DAEDL effectively quantifies uncertainty for OOD data, whereas EDL fails to quantify uncertainty for testing examples distant from the training data, as illustrated in \cref{Fig1}. 

Second, we prove that DAEDL can be interpreted as predicting an \textit{input-dependent posterior distribution of Dirichlet-Categorical model} \cite{charpentier2020posterior} with an improper prior $\boldsymbol{\pi} \sim \Dir (\mathbf{0})$ in the Bayesian context (\cref{Bayesian Interpretation}). This indicates that DAEDL addresses the challenge of prior specification by using an improper prior, enabling the model to learn the Dirichlet distribution solely from the data. 

Third, we prove that the predictive distribution of DAEDL can be conceptualized as an \textit{adaptive temperature scaled} \cite{joy2023sample} softmax model (\cref{Relationship with AdaTS}). This indicates that DAEDL inherits the advantages of adaptive temperature scaling, which is known for its effectiveness in enhancing model calibration and OOD detection performance \cite{balanya2022adaptive, joy2023sample, krumpl2024ats}. 

Finally, we prove that the predictive uncertainty of a testing example estimated using DAEDL is proportional to the distance between the testing example and training data manifold in both the feature space (\cref{Feature Distance Awareness}) and input space  (\cref{Input Distance Awareness}) under mild conditions. This property is formally defined as \textit{distance awareness} \cite{liu2020simple} and has been established as a beneficial condition for obtaining high-quality uncertainty estimates. 

\begin{theorem}
\label{Uniform for OOD Data}
\textit{\textbf{(Uniform Prediction for OOD Data) }} As the distance between the testing example $\mathbf{x}_{\ood}^{\star}$ and training data in the input space diverges, i.e., $\mathbb{E}_{\mathbf{x}' \sim \mathcal{X}_{\tr}} \lVert \mathbf{x}_{\ood}^{\star}-\mathbf{x}' \rVert_{2} \rightarrow \infty$, the predictive distribution of DAEDL converges to the uniform distribution over classes $\forall {c} \in [C]$, i.e., $p(y|\mathbf{x}_{\ood}^{\star}) \rightarrow \mathcal{U}\{1,C\}$. Additionally, the concentration parameters of DAEDL converges to $\mathbf{1}$, i.e., $\boldsymbol{\alpha}(\mathbf{x}_{\ood}^{\star}) \rightarrow \mathbf{1}$.
\end{theorem}
\begin{theorem}
\label{Bayesian Interpretation}
\textbf{\textit{(Bayesian Interpretation of DAEDL)}}
In the Bayesian context, DAEDL can be interpreted as predicting an input-dependent posterior distribution of the Dirichlet-Categorical model with an improper prior $\boldsymbol{\pi} \sim \Dir(\mathbf{0})$.
\end{theorem}
\begin{theorem}
\label{Relationship with AdaTS}
\textbf{\textit{(Relationship with Temperature Scaling)}} The predictive distribution of DAEDL aligns with the adaptive temperature scaled softmax model: 
\begin{align*}
p(y|\mathbf{x}^{\star}) = \Cat(\boldsymbol{\bar{\pi}}), \quad 
\bar{\boldsymbol{\pi}} = \sigma \left( {\mathbf{z}(\mathbf{x}^{\star})}/{T(\mathbf{x}^{\star})} \right),
\end{align*}
where $\mathbf{z}(\mathbf{x}^{\star})=g_{\hat{\boldsymbol{\phi}}}(f_{\hat{\boldsymbol{\theta}}}(\mathbf{x}^{\star}))$ is the logits and  $T(\mathbf{x}^{\star}) = 1/s(\mathbf{x}^{\star})$ is a sample-dependent temperature. 
\end{theorem} 
\begin{theorem}
\label{Feature Distance Awareness}
\textbf{\textit{(Feature Distance Awareness)}} The predictive distribution of the testing example $\mathbf{x}^{\star}$ obtained by DAEDL is \textit{distance aware} in the feature space, i.e., $u(\mathbf{z}^{\star}) = \nu \big(\mathbb{E}_{\mathbf{z}' \sim \mathcal{Z}_{\tr}} \lVert \mathbf{z}^{\star} - \mathbf{z}' \rVert_{2} \big)$, where $u$ is an uncertainty measure, $\nu$ is a monotonic function, $\mathbf{z}^{\star} = f_{\boldsymbol{\theta}}(\mathbf{x}^{\star})$ is the feature representation of $\mathbf{x}^{\star}$, and $\mathcal{Z}_{tr}$ is the set of feature representations of the training data.
\end{theorem} 
\begin{corollary} \textbf{(Input Distance Awareness)}
\label{Input Distance Awareness}
If $f_{\boldsymbol{\theta}}$ is constructed using residual blocks (e.g., ResNet), the predictive distribution of $\mathbf{x}^{\star}$ obtained by DAEDL is distance aware in the input space, i.e., $u(\mathbf{x}^{\star}) = \nu \big(\mathbb{E}_{\mathbf{x}' \sim \mathcal{X}_{\tr}} \lVert \mathbf{x}^{\star} - \mathbf{x}' \rVert_{2} \big)$.
\end{corollary}
Assumptions, proofs, and more detailed explanations of the theoretical results are provided in \cref{Appendix : Detailed Explanation for Theoretical Analysis}. 

\section{Related Works}
\textbf{DBU models.} Distinctions among DBU models have emerged in various aspects, including parameterization, the requirement of OOD data, loss functions, and regularizers.
{KL-PN} \cite{malinin2018predictive} is trained to minimize simultaneously the KL divergence towards a peaked Dirichlet distribution for ID data and a uniform Dirichlet distribution for OOD data. {RKL-PN} \cite{malinin2019reverse} employs the reverse KL divergence instead, arguing that using the KL divergence results in a multimodal target distribution, leading to undesirable uncertainty representation. \citet{nandy2020towards} argued that RKL-PN struggles to distinguish ID data with high aleatoric uncertainty \cite{malinin2018predictive} from OOD data and proposed a novel loss function to maximize the representation gap between ID and OOD data. However, these models require OOD data for training, which is often an unrealistic assumption in practice. 

The Posterior Network ({PostNet}) \cite{charpentier2020posterior} predicts the posterior Dirichlet distribution by utilizing feature space density estimated through Normalizing Flow \cite{rezende2015variational}. The Natural Posterior Network ({NatPN}) \cite{charpentier2021natural} extends the PostNet to arbitrary distributions within the exponential family. However, because these models heavily rely on the feature space density for uncertainty quantification, they may encounter difficulties in practical scenarios where obtaining high-quality feature space is not always feasible.

{EDL} \cite{sensoy2018evidential}, as discussed in \cref{Evidential Deep Learning}, is another notable instance of the DBU model. \citet{haussmann2019bayesian} proposed a Bayesian version of EDL trained with the marginal likelihood regularized by the Probably Approximately Correct bound regularizer. Additionally, \citet{tsiligkaridis2021information} suggested using the $l_{p}$ loss combined with Rényi divergence regularizer. {$\mathcal{I}$-EDL} \cite{deng2023uncertainty} incorporated Fisher information to weigh the importance of each class during training, demonstrating significant performance gains over EDL.
Moreover, alternative DBU models exist, adopting artificial OOD data generation for training  \cite{sensoy2020uncertainty,hu2021multidimensional}, knowledge distillation \cite{malinin2019ensemble, fathullah2022self}, or posterior Dirichlet distribution prediction by variational inference \cite{chen2018variational, joo2020being}.  However, the uncertainty quantified by these models may fail to reflect the distance between the testing example and the training data, hindering their effectiveness in OOD detection. 

\textbf{Other single forward pass uncertainty models.} Alternative models capable of quantifying uncertainty in a single forward pass typically involve i) regularization of the feature space, and ii) uncertainty estimation.  
To obtain the regularized feature space, spectral normalization \cite{liu2020simple, van2021feature, kotelevskii2022nonparametric, mukhoti2023deep} and gradient penalty \cite{van2020uncertainty} have been used. Subsequently, Gaussian process \cite{liu2020simple, van2021feature}, Radial Basis Function network \cite{van2020uncertainty}, GDA \cite{mukhoti2023deep}, and kernel density estimator \cite{kotelevskii2022nonparametric} have been used for estimating uncertainty in a single forward pass. However, these models often require substantial modifications to the model structure or additional computational costs for uncertainty estimation.

\section{Experiments}
\label{Experiments}
\subsection{Experimental Settings}
\textbf{Tasks.} We conducted extensive experiments across various downstream tasks related to uncertainty estimation and classification. Our primary goal was to evaluate whether DAEDL successfully addressed the limitations of conventional EDL, thereby enhancing uncertainty estimation and classification performance. Additionally, we aimed to empirically validate the theoretical advancements presented in \cref{Theoretical Analysis}. The specific questions explored through the experiments are outlined below, with the most relevant properties listed in parentheses. 
\begin{itemize}
\item \textit{Q1}. Does DAEDL improve uncertainty estimation by leveraging feature space density?  {(\cref{Uniform for OOD Data})}
\item \textit{Q2}. Does DAEDL improve classification by using the new parameterization? {(\cref{Bayesian Interpretation})} 
\item \textit{Q3}. Does DAEDL improve confidence calibration? {(\cref{Relationship with AdaTS})} 
\item \textit{Q4}. Does the uncertainty quantified by DAEDL reflect the distance between the testing example and training data? {(\cref{Feature Distance Awareness}, \cref{Input Distance Awareness})}
\end{itemize}
To address \textit{Q1}, we performed {OOD detection (\cref{Exp : OOD Detection})} as a downstream task to evaluate the quality of uncertainty estimation. \textit{Q2} was evaluated by an {image classification (\cref{Exp : Classification and Conf Calibration})} task. For \textit{Q3}, a {confidence calibration (\cref{Exp : Classification and Conf Calibration})} task was executed in two manners: i) conducting a misclassified image detection task, and ii) measuring the Brier score. To address \textit{Q4}, we conducted {distribution shift detection (\cref{Exp : Distribution Shift Detection})}. Similar to \textit{Q1}, OOD detection was applied as a downstream task to assess the quality of the uncertainty estimate. Specifically, we systematically performed a series of OOD detection tasks using progressively generated OOD datasets. These datasets were created by applying various types of corruption to ID data, with the severity of corruption increasing sequentially. Note that each experiment not only evaluates the corresponding question but also relates to other questions. For instance, OOD detection is also related to \textit{Q3} and \textit{Q4}. For each experiment, the mean and standard deviation of the results averaged over five runs were reported. The code for our model is available at \href{https://github.com/TaeseongYoon/DAEDL}{https://github.com/TaeseongYoon/DAEDL}. 

\textbf{Datasets.} To evaluate the OOD detection performance, we used {MNIST} \cite{lecun1998mnist} and {CIFAR-10} \cite{krizhevsky2009learning} as ID datasets. We used {FMNIST} \cite{xiao2017fashion} and {KMNIST} \cite{clanuwat2018deep} as OOD datasets for MNIST. For CIFAR-10, we used {SVHN} \cite{netzer2011reading} and {CIFAR-100} \cite{krizhevsky2009learning} as OOD datasets. We evaluated the classification accuracy and confidence calibration performance of our model using MNIST and CIFAR-10. To evaluate the performance of distribution shift detection, we employed {MNIST-C} \cite{mu2019mnist} and {CIFAR-10-C} \cite{hendrycks2019benchmarking} as the OOD dataset for MNIST and CIFAR-10, respectively. MNIST-C and CIFAR-10-C are datasets created by applying continuous distribution shifts to MNIST and CIFAR-10, respectively. More detailed descriptions of the datasets are provided in   \cref{Appendix : Experimental Details_Datasets}. 
  
\textbf{Implementation.} For a fair comparison, we followed the settings of \citet{charpentier2020posterior} and \citet{deng2023uncertainty}. We used 3 convolutional layers and 3 dense layers for MNIST, and used VGG-16 \cite{simonyan2014very} for CIFAR-10. The optimal hyperparameters were determined through grid search. Additionally, we used a learning rate scheduler and early stopping based on the validation loss. More detailed explanations of the implementation are provided in \cref{Appendix : Experimental Details_ Implementation Details}.

\begin{table*}[htbp]
\centering
\caption{AUPR scores of OOD detection based on aleatoric and epistemic uncertainty. {A} $\rightarrow$ {B} denotes that {A} is employed as an ID dataset, while {B} is utilized as an OOD dataset. {``ALEA."} and {``EPIS."} indicate that the results were obtained by employing aleatoric and epistemic uncertainty measures as an OOD score, respectively. The first four lines, excluding the results of CIFAR-10 $\rightarrow$ CIFAR-100, were obtained from \citet{charpentier2020posterior}. The remaining results were obtained from \citet{deng2023uncertainty}.}
\vspace{0.15in}
\label{OOD Detection}
\begin{small}
\begin{sc}
\resizebox{\columnwidth * 2}{!}{%
\begin{tabular}{@{}lcccccccc@{}}
\toprule
 &
  \multicolumn{2}{c}{\textbf{MNIST $\rightarrow$ KMNIST}} &
  \multicolumn{2}{c}{\textbf{MNIST $\rightarrow$ FMNIST}} & 
  \multicolumn{2}{c}{\textbf{CIFAR-10 $\rightarrow$ SVHN}} &
  \multicolumn{2}{c}{\textbf{CIFAR-10 $\rightarrow$ CIFAR-100}} \\ \midrule
& \textbf{Alea.}   & \textbf{Epis.}   & \textbf{Alea.}    & \textbf{Epis.}   & \textbf{Alea.}    & \textbf{Epis.}   & \textbf{Alea.}   & \textbf{Epis.}  \\ \midrule
Dropout  & 94.00 $\pm$ 0.1 & \quad -  & 96.56  $\pm$ 0.2 & \quad -  & 51.39 $\pm$ 0.1 & \quad - & 45.57 $\pm$ 1.0 & \quad - \\
KL-PN  & 92.97 $\pm$ 1.2 & 93.39  $\pm$ 1.0 & 98.44  $\pm$ 1.0 & 98.16  $\pm$ 0.0 & 43.96 $\pm$ 1.9  & 43.23 $\pm$  2.3 & 61.41 $\pm$ 2.8 & 61.53 $\pm$ 3.4 \\
RKL-PN & 60.76 $\pm$ 2.9 & 53.76  $\pm$ 3.4 & 78.45  $\pm$ 3.1 & 72.18 $\pm$ 3.6 & 53.61 $\pm$ 1.1 & 49.37 $\pm$  0.8 & 55.42 $\pm$ 2.6 & 54.74 $\pm$ 2.8 \\
PostNet & 95.75 $\pm$ 0.2 & 94.59  $\pm$ 0.3 & 97.78  $\pm$ 0.2 & 97.24  $\pm$ 0.3 & 80.21 $\pm$ 0.2  & 77.71 $\pm$  0.3 & 81.96 $\pm$ 0.8 & 82.06 $\pm$ 0.8 \\
EDL & 97.02 $\pm$ 0.8 & 96.31  $\pm$ 2.0 & 98.10  $\pm$ 0.4 & 98.08  $\pm$ 0.4 & 78.87 $\pm$  3.5 & 79.12 $\pm$ 3.7  & 84.30 $\pm$ 0.7 & 84.18 $\pm$ 0.7 \\
$\mathcal{I}$-EDL & 98.34 $\pm$ 0.2 & 98.33 $\pm$ 0.2 & 98.89  $\pm$ 0.3 & 98.86  $\pm$ 0.3 & 83.26 $\pm$ 2.4  & 82.96 $\pm$ 2.2  & 85.35 $\pm$ 0.7 & 84.84 $\pm$ 0.6 \\ \midrule
\textbf{DAEDL} & \textbf{99.90 $\pm$ 0.0} & \textbf{99.92 $\pm$ 0.0} & \textbf{99.83 $\pm$ 0.0} & \textbf{99.87 $\pm$ 0.0}  & \textbf{85.50 $\pm$ 1.4} & \textbf{85.54 $\pm$ 1.4} & \textbf{88.16 $\pm$ 0.1} & \textbf{88.19 $\pm$ 0.1} \\ \bottomrule
\end{tabular}
}
\end{sc}
\end{small}
\end{table*}

\textbf{Baselines.} We compared our model with representative DBU models. Following the setup of \citet{deng2023uncertainty}, our evaluation included the following competing models: {KL-PN} \cite{malinin2018predictive}, {RKL-PN} \cite{malinin2019reverse}, {PostNet} \cite{charpentier2020posterior}, {EDL} \cite{sensoy2018evidential}, and $\mathcal{I}$-{EDL} \cite{deng2023uncertainty}. In \cref{Exp : OOD Detection} and \cref{Exp : Classification and Conf Calibration}, we also included a comparison with {Dropout} \cite{gal2016dropout}, which still demonstrates state-of-the-art uncertainty estimation performance in various tasks. In \cref{Exp : Distribution Shift Detection}, we further compared our model with {MSP} \cite{hendrycks2016baseline}, a standard baseline for distribution shift detection. Following \citet{charpentier2020posterior} and \citet{deng2023uncertainty}, we used the set of data points generated by uniform noise as an OOD dataset for KL-PN and RKL-PN, which require an OOD dataset for training, to ensure a fair comparison.

\subsection{OOD Detection}
\label{Exp : OOD Detection}
We evaluated the OOD detection performance of DAEDL in comparison to baseline methods. Our evaluation consisted of two steps. First, we calculated OOD scores for both the ID test dataset and the OOD dataset using specific uncertainty measures. Following \citet{charpentier2020posterior} and \citet{deng2023uncertainty}, for DBU models, we used $\underset{c}{\max} \{\mathbb{E}_{\boldsymbol{\pi} \sim \Dir(\boldsymbol{\alpha})}[\pi_{c}]\}$ (i.e., maximum expected class probability) to measure aleatoric uncertainty, while we used $\alpha_{0}$ (i.e., precision of the Dirichlet distribution) as the measure of epistemic uncertainty. For {Dropout}, we adopted $\max_{c} \pi_{c}$ (i.e., maximum class probability) as the aleatoric uncertainty measure. Second, using the calculated OOD scores, we computed the area under the precision-recall curve {(AUPR)} score, assigning label 1 to ID and 0 to OOD data, to evaluate the OOD detection performance. 

Based on the AUPR scores reported in \cref{OOD Detection}, DAEDL demonstrates state-of-the-art performance across all evaluated tasks. Specifically, DAEDL outperformed the runner-up method ($\mathcal{I}$-EDL) by noteworthy margins, achieving improvements of {1.56}, {1.59}, {0.94}, and {1.01} on MNIST, as well as {2.24}, {2.58}, {2.81}, and {3.35} on CIFAR-10. OOD detection results with an additional performance metric, the area under the receiver operating characteristic curve (AUROC), are provided in \cref{Appendix : OOD Detection}.

\subsection{Image Classification \& Confidence Calibration}
\label{Exp : Classification and Conf Calibration}
We conducted both image classification and confidence calibration tasks using DAEDL and the baseline methods. For the image classification task, we evaluated performance using the test accuracy. For the confidence calibration task, we evaluated performance in two different ways. 
First, we evaluated the misclassified image detection performance using the AUPR score. Specifically, we first split the ID test dataset into two groups based on the classification results: one with correctly classified data and the other with misclassified data. Then, we calculated the confidence score for each group using a specific confidence measure: $\underset{c}{\max} \{\mathbb{E}_{\boldsymbol{\pi} \sim \Dir(\boldsymbol{\alpha})}[\pi_{c}]\}$ for DBU models and $\max_{c} \pi_{c}$ for Dropout, following \citet{charpentier2020posterior} and \citet{deng2023uncertainty}. Using the calculated confidence scores, we computed the AUPR scores with label 1 for correctly classified data and 0 for misclassified ones, to evaluate the misclassified image detection performance. 
Second, we measured the {Brier score} \cite{brier1950verification}, which is a standard metric used to assess the calibration of the model \cite{gneiting2007strictly}. A lower value of the Brier score indicates better performance. 

As shown in \cref{Confidence Calibration}, DAEDL exhibits state-of-the-art performance in image classification, outperforming the runner-up method ($\mathcal{I}$-EDL) by a significant margin of {1.91}. Moreover, DAEDL achieved state-of-the-art performance in confidence calibration, surpassing the respective runner-up methods, $\mathcal{I}$-EDL and PostNet, by {0.36} in the AUPR score and by {8.57} in the Brier score.
Despite the primary focus of the DBU models on uncertainty estimation, achieving high classification accuracy remains crucial. Therefore, DAEDL holds a distinct advantage by achieving the best performance in both classification and confidence calibration.

\begin{table}[htbp] 
\caption{The results of image classification and confidence calibration on CIFAR-10. 
The first four lines present the results from \citet{charpentier2020posterior}. The test accuracy and AUPR of EDL and $\mathcal{I}$-EDL are obtained from \citet{deng2023uncertainty}. 
}
\vskip 0.15in
\label{Confidence Calibration}
\begin{small}
\begin{center}
\begin{sc}
\begin{tabular}{lccc}
\toprule
& \textbf{Test Acc.} & \textbf{AUPR} &  \textbf{Brier} \\ \midrule
Dropout & 82.84  $\pm$ 0.1 & 97.15 $\pm$ 0.0 &  27.15 $\pm$ 0.2 \\
KL-PN   & 27.46 $\pm$ 1.7  & 50.61 $\pm$ 4.0 & 87.28 $\pm$ 1.0 \\
RKL-PN  & 64.76 $\pm$ 0.3  & 86.11 $\pm$ 0.4 & 54.73 $\pm$ 0.4 \\
PostNet & 84.85 $\pm$ 0.0  & 97.76 $\pm$ 0.0 & 22.84 $\pm$ 0.0 \\
EDL     & 83.55 $\pm$ 0.6  & 97.86 $\pm$ 0.2 & 33.38 $\pm$ 2.0 \\
$\mathcal{I}$-EDL   & 89.20 $\pm$ 0.3  & 98.72 $\pm$ 0.1 & 35.20 $\pm$ 0.8 \\ \midrule
\textbf{DAEDL} & \textbf{91.11 $\pm$ 0.2}  & \textbf{99.08 $\pm$ 0.0} &  \textbf{14.27 $\pm$ 0.2} \\ \bottomrule
\end{tabular}
\end{sc}
\end{center}
\end{small}
\vskip -0.15in
\end{table}

\begin{table*}[htbp]
\caption{AUPR scores of distribution shift detection based on aleatoric uncertainty. $\mathcal{C} \in \{1,2,3,4,5\}$ denotes the severity level of the corruptions in CIFAR-10-C. The results are averaged over 19 different corruptions for each severity level.}
\label{Distribution shift detection}
\vskip 0.15in
\begin{center}
\begin{small}
\begin{sc}
\begin{tabular}{@{}l|l|ccccc@{}}
\toprule
 & \textbf{MNIST $\rightarrow$ MNIST-C} & \multicolumn{5}{c} {\textbf{CIFAR-10 $\rightarrow$ CIFAR-10-C}} \\ \midrule
 & & \quad \  $\mathcal{C} = 1$ & \quad $\mathcal{C} = 2$ & \quad \  $\mathcal{C} = 3$ & \quad $\mathcal{C} = 4$ & \quad \ $\mathcal{C} = 5$ \\ \midrule
MSP & \qquad \ \ 78.54 $\pm$ 0.3 & 56.39 $\pm$ 0.7 & 61.88 $\pm$ 1.1 & 65.86 $\pm$ 1.3 & 69.91 $\pm$ 1.5 & 75.01 $\pm$ 1.8  \\ 
EDL & \qquad \ \ 82.75 $\pm$ 0.8 & 54.76 $\pm$ 0.3 & 59.01 $\pm$ 0.4 & 62.46 $\pm$ 0.5 & 65.87 $\pm$ 0.6 & 70.21 $\pm$ 0.8  \\
$\mathcal{I}$-EDL & \qquad \ \ 86.06 $\pm$ 0.5 & 56.33 $\pm$ 0.2 & 61.52 $\pm$ 0.5 & 65.44 $\pm$ 0.5 & 69.45 $\pm$ 0.5 & 74.56 $\pm$ 0.5  \\ \midrule
\textbf{DAEDL} & \qquad \ \ \textbf{92.43 $\pm$ 0.3} & \textbf{57.89 $\pm$ 0.3} & \textbf{63.23 $\pm$ 0.4} & \textbf{67.53 $\pm$ 0.4} & \textbf{72.21 $\pm$ 0.4} & \textbf{77.74 $\pm$ 0.4} \\ \bottomrule
\end{tabular}
\end{sc}
\end{small}
\end{center}
\vskip -0.15in
\end{table*}

\begin{table*}[htbp] 

\caption{Ablation study results on CIFAR-10. {``AUPR"} represents the performance of misclassified image detection on the CIFAR-10 dataset. The results of EDL (DAEDL without EXP, DE, and SN) are from \citet{deng2023uncertainty}}

\label{Ablation Study - CIFAR10}
\vskip 0.15in
\begin{center}
\begin{small}
\begin{sc}
\begin{tabular}{@{}lcc|cccccccc@{}}
\toprule
 &  &  &  &  & \multicolumn{2}{c}{\textbf{CIFAR10 $\rightarrow$ SVHN}} & \multicolumn{2}{c}{\textbf{CIFAR-10 $\rightarrow$ CIFAR-100}} \\ \midrule
\textbf{Exp} & \textbf{DE} & \textbf{SN} & \textbf{Test.Acc} & \textbf{AUPR} & \textbf{Alea.} & \textbf{Epis.} & \textbf{Alea.} & \textbf{Epis.} \\ \midrule
\xmark & \xmark & \xmark & 83.55 $\pm$ 0.6 & 97.86 $\pm$ 0.2 & 78.87 $\pm$ 3.5 & 79.12 $\pm$ 3.7 & 84.30  $\pm$ 0.7 & 84.18 $\pm$ 0.7 \\ \midrule
\cmark & \xmark & \xmark & 88.59 $\pm$ 0.4 & 98.42 $\pm$ 0.1& 80.39 $\pm$ 2.0 & 80.45 $\pm$ 1.9 & 83.62 $\pm$ 0.9 & 83.67 $\pm$ 0.9  \\
\cmark & \cmark & \xmark & 88.59 $\pm$ 0.4 & 99.01 $\pm$ 0.0 & 85.02 $\pm$ 0.7  & 85.04 $\pm$ 0.7 & 87.48 $\pm$ 0.2  & 87.50 $\pm$ 0.1 \\ 
\cmark & \xmark & \cmark &  \textbf{91.11 $\pm$ 0.2} & 99.04 $\pm$ 0.0
 & 84.53 $\pm$ 0.9  & 84.55 $\pm$ 0.9 & 87.52 $\pm$ 0.2  & 87.54 $\pm$ 0.2  \\ \midrule
\cmark & \cmark & \cmark & \textbf{91.11 $\pm$ 0.2} & \textbf{99.08 $\pm$ 0.0}
& \textbf{85.50 $\pm$ 1.4} & \textbf{85.54 $\pm$ 1.4} & \textbf{88.16 $\pm$ 0.1} & \textbf{88.19 $\pm$ 0.1} \\ \bottomrule
\end{tabular} 
\end{sc}
\end{small}
\end{center}
\vskip -0.15in
\end{table*}

\subsection{Distribution Shift Detection}
\label{Exp : Distribution Shift Detection}
We conducted distribution shift detection using an OOD dataset created by applying distribution shift (corruption) to the ID dataset. We used the static MNIST-C \cite{mu2019mnist} and CIFAR-10-C \cite{hendrycks2019benchmarking} datasets as OOD datasets, paired with the MNIST and CIFAR-10 datasets as ID datasets, respectively. The static MNIST-C dataset has a fixed severity level of corruption, whereas the CIFAR-10-C dataset has five severity levels of corruption. In the CIFAR-10-C experiments, our objective was to assess the performance of DAEDL in detecting both the presence of corruption and its severity. We measured the performance using AUPR.

As shown in \cref{Distribution shift detection}, DAEDL achieved state-of-the-art performance across all tasks. Specifically, DAEDL outperformed the runner-up method (MSP) by {6.37} in the MNIST-C experiments, and by {1.50}, {1.35}, {1.67}, {2.30}, and {2.73} in the CIFAR-10-C experiments for the severity level $\mathcal{C}$ of 1, 2, 3, 4, and 5, respectively. Notably, the performance of EDL and $\mathcal{I}$-EDL was even worse than MSP, which is based on the softmax model, known for its inefficacy in quantifying uncertainty.  
This inferior performance of EDL models may arise from the unique challenges of distribution shift detection. In typical OOD detection tasks, large distances between OOD testing examples and training data allow for effective discrimination, even if the uncertainty estimates do not accurately reflect the distance. However, in distribution shift detection, OOD data undergo a distribution shift from ID data, with only minor distances from the ID data. Therefore, precise uncertainty estimation is particularly crucial for successful distribution shift detection. 
DAEDL effectively conducted distribution shift detection by leveraging the feature space density of the testing example to account for its distance from the training data. More experimental results, including additional uncertainty measures, performance measures, and graphical representations of the results for each corruption scenario, are provided in \cref{Appendix : Distribution Shift Detection}. 

\subsection{Ablation Study}
\vskip -0.04in
\label{Exp : Ablation Study}
We conducted an ablation study of DAEDL on CIFAR-10 to evaluate the contributions of its key components: i) alternative parameterization ({EXP}), ii) feature space density integration ({DE}), and iii) spectral normalization ({SN}). We conducted OOD detection, image classification, and misclassified image detection, using the same OOD datasets and procedures described in \cref{Exp : OOD Detection} and \cref{Exp : Classification and Conf Calibration}. The results, presented in \cref{Ablation Study - CIFAR10}, demonstrate that all components of DAEDL significantly contributed to enhancing overall performance. 
Specifically, EDL equipped with the proposed parameterization achieved significantly higher classification accuracy, while maintaining comparable uncertainty estimation performance to that of EDL. Moreover, the integration of feature space density resulted in a significant enhancement in uncertainty estimation performance. Furthermore, the application of spectral normalization resulted in an overall improvement in DAEDL's performance. The ablation study results for MNIST are provided in \cref{Appendix : Ablation Study}. 

\section{Conclusion}
We proposed a novel method, DAEDL, which improves the classification and OOD detection performance of EDL. DAEDL achieves this improvement by incorporating the feature space density of testing examples during prediction and introducing a new parameterization of the concentration parameters of the Dirichlet distribution. We demonstrated the effectiveness of DAEDL both theoretically and empirically, showcasing its favorable theoretical properties and state-of-the-art performance in various downstream tasks related to uncertainty estimation and classification. A potential future research direction includes extending DAEDL for regression tasks.

\section*{Acknowledgements}
This work was supported by the National Research Foundation of Korea (NRF) grant funded by the Korea government (MSIT) (2023R1A2C2005453, RS-2023-00218913).

\section*{Impact Statement} 
Accurate uncertainty estimation is crucial for ensuring the safe deployment of AI models, particularly in high-risk domains such as healthcare, finance, and manufacturing. However, existing uncertainty estimation models encounter practical challenges, such as the need for multiple forward passes, substantial modifications to the neural network structure, limited classification performance, and sensitive hyperparameters. In this paper, we introduce DAEDL, a novel approach capable of producing high-quality uncertainty estimates in a single forward pass. DAEDL is comprised of detachable components easily integrable into existing network structures, and excels in classification tasks while operating with non-sensitive hyperparameters. Moreover, DAEDL demonstrates remarkable performance in OOD and distribution shift detection, highlighting its potential impact in addressing real-world challenges effectively. 
Despite its promising advancements, it is essential to acknowledge the limitations of DAEDL. In complex scenarios, EDL outputs and the feature space density may not accurately capture uncertainty, potentially limiting DAEDL's effectiveness. Therefore, practitioners must carefully assess the suitability of DAEDL for their specific problem domains before solely relying on its outcomes.
\bibliography{daedl}
\bibliographystyle{icml2024}
\newpage
\onecolumn
\appendix
\icmltitle{Appendix for the Paper \\ ``Uncertainty Estimation by Density Aware Evidential Deep Learning"}
\section{Additional Explanation for Theoretical Analysis}
\label{Appendix : Detailed Explanation for Theoretical Analysis}
In this section, we provide an additional explanation about the theoretical analysis (\cref{Theoretical Analysis}) of DAEDL. First, we outline the assumptions and lemmas that are utilized throughout the study. Second, we present detailed proofs of the theorems. Third, we provide an additional description of the Bayesian interpretation of DAEDL, which is established in \cref{Bayesian Interpretation}. Finally, we offer additional insights into  \cref{Relationship with AdaTS}, clarifying the relationship between EDL, AdaTS, and DAEDL. 
Here, $\mathcal{X}_{\tr}$ and $\mathcal{Z}_{tr} =\{\mathbf{x} = f_{\hat{\boldsymbol{\theta}}}(\mathbf{x}), \ \forall \mathbf{x} \in \mathcal{X}_{\tr}\}$ denote the sets of training data and their feature representations. In addition, $\mathbf{x}^{\star}$ and $\mathbf{z}^{\star} = f_{\hat{\boldsymbol{\theta}}}(\mathbf{x}^{\star})$ represent the testing example and its feature representation, respectively. 

\subsection{Assumptions \& Lemmas}
We state the assumptions that are required to prove the theorems. 
First, we assume that if the distance between the testing example and the training data diverges in the input space, it will also diverge in the feature space (\cref{assumption1}). \cref{assumption1} is used to prove \cref{Uniform for OOD Data}. Second, we assume that the normalized feature space density (i.e., $s(\mathbf{x}^{\star})$) is monontonically decreasing with respect to the distance between the testing example and the training data in the feature space (i.e., $\mathbb{E}_{\mathbf{z}' \sim \mathcal{Z}_{\tr}} \lVert \mathbf{z}^{\star} - \mathbf{z}' \rVert_{2}$) (\cref{assumption2}). \cref{assumption2} is used to prove \cref{Feature Distance Awareness} and \cref{Input Distance Awareness}. 
\begin{assumption}
\label{assumption1}
If $\mathbb{E}_{\mathbf{x}' \sim \mathcal{X}_{\tr}} \lVert \mathbf{x}^{\star} - \mathbf{x}' \rVert_{2} \rightarrow  \infty$, then $\mathbb{E}_{\mathbf{z}' \sim \mathcal{Z}_{\tr}} \lVert \mathbf{z}^{\star} - \mathbf{z}' \rVert_{2} \rightarrow \infty$. 
\end{assumption}
\begin{assumption}
\label{assumption2}
$s(\mathbf{x}^{\star})$ is monotonically decreasing with respect to $\mathbb{E}_{\mathbf{z}' \sim \mathcal{Z}_{\tr}} \lVert \mathbf{z}^{\star} - \mathbf{z}' \rVert_{2}$.
\end{assumption}

We provide the lemmas that are required to prove the theorems.
\begin{lemma}
\label{lemma1}
If spectral normalization is applied to $f_{\boldsymbol{\theta}}$, then $\mathbb{E}_{\mathbf{z}' \sim \mathcal{Z}_{tr}} \lVert \mathbf{z}^{\star}-\mathbf{z}' \rVert_{2}$ is bounded by $\mathbb{E}_{\mathbf{x}' \sim \mathcal{X}_{tr}} \lVert \mathbf{x}^{\star}-\mathbf{x}' \rVert_{2}$. 
\end{lemma}
\begin{proof}
If spectral normalization is applied to $f_{\boldsymbol{\theta}}$, $f_{
\boldsymbol{\theta}}$ is 1-Lipschitz continuous \cite{miyato2018spectral}. In other words, for data points $ \mathbf{x}_{1}, \mathbf{x}_{2} \in \mathbb{R}^{D}$ and their corresponding  feature representations $\mathbf{z}_{1}, \mathbf{z}_{2} \in \mathbb{R}^{H}$, where $\mathbf{z}_{1} =  f_{\boldsymbol{\theta}}(\mathbf{x}_{1})$ and $\mathbf{z}_{2} = f_{\boldsymbol{\theta}}(\mathbf{x}_{2})$, the inequality $\lVert \mathbf{z}_{1} - \mathbf{z}_{2} \rVert_{2} \leq \lVert \mathbf{x}_{1} - \mathbf{x}_{2} \rVert_{2}$ holds. 
By generalizing this inequality, we obtain :
\begin{align*}
\mathbb{E}_{\mathbf{z}' \sim \mathcal{Z}_{tr}}\lVert \mathbf{z}^{\star}-\mathbf{z}' \rVert_{2} \leq  \mathbb{E}_{\mathbf{x}' \sim \mathcal{X}_{tr}}\lVert \mathbf{x}^{\star}-\mathbf{x}' \rVert_{2}. 
\end{align*}    
\end{proof}

\begin{lemma} \label{lemma2} (Lemma 5 of \citet{charpentier2021natural}) Let $p(\mathbf{z}^{\star};\hat{\boldsymbol{\alpha}})$ be parameterized with a Gaussian Mixture Model (GMM). Then $\mathbb{E}_{\mathbf{z}' \sim \mathcal{Z}_{tr}} \lVert \mathbf{z}^{\star}-\mathbf{z}' \rVert_{2} \rightarrow \infty$ implies that $p(\mathbf{z}^{\star} ;\hat{\boldsymbol{\alpha}}) \rightarrow 0$.
\end{lemma}

\begin{lemma} \label{lemma3} (Modified from Proposition 5.5. of \citet{bui2023density}) $u(\mathbf{z}^{\star})$ is monotonically decreasing with respect to $s(\mathbf{x}^{\star})$ on the interval $(0,1]$. 
\end{lemma}
\begin{proof}
Suppose that we employ the predictive entropy (i.e., $u(\mathbf{z}^{\star}) = \mathcal{H}[p(y|\mathbf{z}^{\star})]$) as the uncertainty measure for DAEDL. Then, the uncertainty of $\mathbf{x}^{\star}$ quantified by DAEDL can be expressed as follows:
$u(\mathbf{z}^{\star}) = \mathcal{H}[\sigma(g_{\hat{\boldsymbol{\phi}}}(\mathbf{z}^{\star}) \times s(\mathbf{x}^{\star})]$. 
This formulation mirrors the uncertainty presented in the proof of Proposition 5.5 in \citet{bui2023density}, differing only in the method for obtaining the normalized feature space density. Therefore, substituting $s(\mathbf{x}^{\star})$ for  $p(\mathbf{z}^{\star};\boldsymbol{\alpha})$ in the proof of the corresponding theorem allows us to conclude the proof.
\end{proof}

\begin{lemma} \label{lemma4} (Proposition 1 of \citet{liu2020simple}) Consider a hidden mapping $f_{\boldsymbol{\theta}} : \mathcal{X} \rightarrow \mathcal{H}$ with residual architecture $f_{\boldsymbol{\theta}} = f_{L-1} \circ \cdots f_{2} \circ f_{1}$, where $f_{l}(\mathbf{x}) = \mathbf{x} + g_{l}(\mathbf{x})$. If for $0 < \alpha \leq 1$, all $g_{l}$'s are $\alpha$-Lipschitz, i.e., $\lVert g_{l}(\mathbf{x}) - g_{l}(\mathbf{x}') \rVert_{\mathcal{H}} \leq \alpha \lVert \mathbf{x}-\mathbf{x}' \rVert_{\mathcal{X}}, \ \ \forall(\mathbf{x},\mathbf{x}') \in \mathcal{X}$. Then, 
\begin{align*}
L_{1}  \lVert \mathbf{x}-\mathbf{x}' \rVert_{\mathcal{X}} \leq \lVert f_{\boldsymbol{\theta}}(\mathbf{x}) - f_{\boldsymbol{\theta}}(\mathbf{x}') \rVert_{\mathcal{H}} \leq L_{2} \lVert \mathbf{x}-\mathbf{x}' \rVert_{\mathcal{X}},
\end{align*}
where $L_{1} = (1-\alpha)^{L-1}$ and $L_{2} = (1+\alpha)^{L-1}$, i.e., $f_{\boldsymbol{\theta}}$ is distance preserving.
\end{lemma}

\subsection{Proofs of Theorems}
We prove \cref{Relationship with AdaTS}, \cref{Uniform for OOD Data}, \cref{Feature Distance Awareness}, and \cref{Input Distance Awareness} sequentially.
\paragraph{Proof for \cref{Relationship with AdaTS}} 
Let $\hat{\boldsymbol{\theta}}$ and $\hat{\boldsymbol{\phi}}$ be the optimal parameters of the feature extractor and the classifier, respectively, which are obtained by training.
For  $\forall c \in [1,2,\cdots C]$, the predictive distribution of testing example $\mathbf{x}^{\star}$ obtained by DAEDL is expressed as follows: 
\begin{equation} 
\begin{split}
p(y=c|\mathbf{x}^{\star};\hat{\boldsymbol{\theta}}, \hat{\boldsymbol{\phi}}) &= \int p(y=c|\boldsymbol{\pi}) p(\boldsymbol{\pi}|\boldsymbol{\alpha}, \mathbf{x}^{\star};\hat{\boldsymbol{\theta}}, \hat{\boldsymbol{\phi}}) d \boldsymbol{\pi} \\
&= \int \pi_{c} \Dir(\boldsymbol{\pi}|\boldsymbol{\alpha}, \mathbf{x}^{\star}; \hat{\boldsymbol{\theta}}, \hat{\boldsymbol{\phi}} ) d\boldsymbol{\pi} \\
&= \mathbb{E}_{\boldsymbol{\pi} \sim \Dir(\boldsymbol{\alpha})}[\pi_{c}]  \\
&= \frac{\alpha_{c}}{\sum_{c=1}^{C} \alpha_{c}}. \label{pred1}
\end{split}
\end{equation}
Additionally, the concentration parameters of DAEDL can be expressed as follows: 
\begin{equation} \label{alpha}
\boldsymbol{\alpha}(\mathbf{x}^{\star}) = \exp (\mathbf{z}(\mathbf{x}^{\star}) / T(\mathbf{x}^{\star})), 
\end{equation}
where $\mathbf{z}(\mathbf{x}^{\star}) = g_{\hat{\boldsymbol{\phi}}}(f_{\hat{\boldsymbol{\theta}}}(\mathbf{x}^{\star}))$ and 
$T(\mathbf{x}^{\star}) = 1 / s(\mathbf{x}^{\star})$.
Plugging Eq.\eqref{alpha} into Eq.\eqref{pred1}, the predictive distribution of $\mathbf{x}^{\star}$ obtained by DAEDL can be expressed as follows:
\begin{align} \label{pred2}
p(y|\mathbf{x}^{\star}) = \Cat(\bar{\boldsymbol{\pi}}), \quad \bar{\boldsymbol{\pi}} = \sigma \left(\mathbf{z}(\mathbf{x}^{\star}) / \ T(\mathbf{x}^{\star}) \right).
\end{align}
Here, we omitted the parameters $\hat{\boldsymbol{\theta}}$ and $\hat{\boldsymbol{\phi}}$ for notational simplicity. From Eq.\eqref{pred2}, we can conclude that the predictive distribution of $\mathbf{x}^{\star}$ obtained by DAEDL aligns with the adaptive temperature scaled softmax model.

\paragraph{Proof for \cref{Uniform for OOD Data}}
First, $\mathbb{E}_{\mathbf{x}' \sim \mathcal{X}_{tr}} \lVert \mathbf{x}_{\ood}^{\star}-\mathbf{x}' \rVert_{2} \rightarrow \infty$ implies that $\mathbb{E}_{\mathbf{z}' \sim \mathcal{Z}_{tr}} \lVert \mathbf{z}_{\ood}^{\star}-\mathbf{z}' \rVert_{2} \rightarrow \infty$ (\cref{assumption1}). Second, $\mathbb{E}_{\mathbf{z}' \sim \mathcal{Z}_{tr}} \lVert \mathbf{z}_{\ood}^{\star}-\mathbf{z}' \rVert_{2} \rightarrow \infty$ implies that $p(\mathbf{z}^{\star}_{\ood} = f_{\hat{\boldsymbol{\theta}}}(\mathbf{z}^{\star}_{\ood})) \rightarrow 0$ (\cref{lemma2}). Therefore, the following inequality holds for $\mathbf{z}^{\star}_{\ood}$ :
\begin{align*}\log p(\mathbf{z}_{\ood}^{\star}) \leq \underset{\mathbf{x} \in \mathcal{X}_{tr}} \min \{\log p(f_{\hat{\boldsymbol{\theta}}}(\mathbf{x})) \}.
\end{align*}
By the definition of the normalizing function $s$, $p(\mathbf{z}^{\star}_{\ood} = f_{\boldsymbol{\theta}}(\mathbf{z}^{\star}_{\ood})) \rightarrow 0$ implies that $s(\mathbf{x}_{\ood}^{\star}) \rightarrow 0$ and $T(\mathbf{x}_{\ood}^{\star}) \rightarrow \infty$. Plugging these into Eq.\eqref{alpha} and Eq.\eqref{pred2}, the concentration parameters and predictive distribution of $\mathbf{x}_{\ood}^{\star}$ obtained by DAEDL are derived as follows:
\begin{align*}
\boldsymbol{\alpha}(\mathbf{x}_{\ood}^{\star}) \rightarrow \textbf{1}, \quad p(y|\mathbf{x}_{\ood}^{\star}) \rightarrow \mathcal{U}\{1,C\}.
\end{align*} 

\paragraph{Proof for \cref{Feature Distance Awareness}}
The predictive distribution of the testing example $\mathbf{x}^{\star}$ obtained by DAEDL can be expressed as follows: 
\begin{align*}
p(y|\mathbf{z}^{\star}) = \sigma(g_{\boldsymbol{\phi}}(\mathbf{z}^{\star}) \times s(\mathbf{x}^{\star})),
\end{align*} 
where $\mathbf{z}^{\star} = f_{\boldsymbol{\theta}}(\mathbf{x}^{\star})$ is the feature representation of $\mathbf{x}^{\star}$.
Suppose that we employ the predictive entropy (i.e., $u(\mathbf{z}^{\star}) = \mathcal{H}[p(y|\mathbf{z}^{\star})]$) as the uncertainty measure for DAEDL. Then,  the uncertainty of $\mathbf{x}^{\star}$ quantified by DAEDL can be expressed as follows: 
\begin{align*}
u(\mathbf{z}^{\star}) = \mathcal{H}[(\sigma(g_{\boldsymbol{\phi}}(\mathbf{z}^{\star}) \times s(\mathbf{x}^{\star})].
\end{align*}
First, $s(\mathbf{x}^{\star})$ is monotonically decreasing with respect to $\mathbb{E}_{\mathbf{z}' \sim \mathcal{Z}_{\tr}} \lVert \mathbf{z}^{\star} - \mathbf{z}' \rVert_{2}$ (\cref{assumption2}). Second, $u(\mathbf{z}^{\star})$ is monotonically decreasing with respect to $s(\mathbf{x}^{\star})$ (\cref{lemma3}). Combining these results, it follows that $u(\mathbf{z}^{\star})$ is monotonically increasing with respect to $\mathbb{E}_{\mathbf{z}' \sim \mathcal{Z}_{\tr}} \lVert \mathbf{z}^{\star} - \mathbf{z}' \rVert_{2}$. In other words, 
\begin{align*}
u(\mathbf{z}^{\star}) = \nu \big(\mathbb{E}_{\mathbf{z}' \sim \mathcal{Z}_{\tr}} \lVert \mathbf{z}^{\star} - \mathbf{z}' \rVert_{2} \big)
\end{align*}
holds for a monotonic function $\nu$. Therefore, the predictive distribution of $\mathbf{x}^{\star}$ obtained by DAEDL is \textit{distance aware} in the feature space. 

\paragraph{Proof for \cref{Input Distance Awareness}}
First, if spectral normalization is applied to $f_{\boldsymbol{\theta}}$, and $f_{\boldsymbol{\theta}}$ is constructed by the residual blocks, $f_{\boldsymbol{\theta}}$  is \textit{distance preserving} (\cref{lemma4}). Second, the uncertainty of $\mathbf{x}^{\star}$ obtained by DAEDL is monotonically increasing with respect to the feature space distance (\cref{Feature Distance Awareness}). In other words, the output layer is \textit{distance aware}. Consequently, the combination of \textit{distance preserving} feature extractor and \textit{distance aware} output layers leads to the \textit{distance awareness} in the input space (Section 2.2 of \citet{liu2020simple}). Therefore, we can conclude that if $f_{\boldsymbol{\theta}}$ is constructed by the residual blocks, the predictive distribution of $\mathbf{x}^{\star}$ obtained by DAEDL is \textit{distance aware} in the input space. 

\subsection{Additional Explanation about \cref{Bayesian Interpretation}}
\label{Appendix : Bayesian Interpretation of DAEDL}
We provide a detailed explanation about \cref{Bayesian Interpretation}. First, we interpret the DBU models in the Bayesian context within the framework of an \textit{input-dependent Dirichlet-Categorical model} \cite{charpentier2020posterior}. Second, we analyze the limitations of the conventional prior specifications utilized in DBU models. Finally, we interpret DAEDL under the same framework to underscore its strength. In particular, DAEDL can be interpreted as predicting a posterior distribution of input-dependent Dirichlet Categorical model using an improper prior $\boldsymbol{\pi} \sim \Dir(\mathbf{0})$. This approach mitigates the challenge of prior specification that occurs in common DBU models, allowing our model to learn the appropriate posterior Dirichlet distribution from the data. 

\paragraph{DBU model under the input-dependent Dirichlet Categorical model framework.}
In a Bayesian context, DBU models can be interpreted under the framework of an \textit{input-dependent Dirichlet-Categorical model} \cite{charpentier2020posterior}. Intuitively, the goal of the DBU model is to predict the posterior Dirichlet distribution for the data using a neural network. The concentration parameters of the prior Dirichlet distribution for these models are determined based on the prior belief about the class counts. When there is no prior information, the parameters of the prior are conventionally set as $\boldsymbol{\alpha}_{\prior}(\mathbf{x}_{i}) = \mathbf{1}$. To estimate the concentration parameters of the posterior Dirichlet distribution, we need to obtain the class counts from the observations. However, in the absence of such observations, DBU models predict \textit{pseudo-observations} $\{\tilde{y}^{(j)}_{i}\}_{j=1}^{N}$ for each data point $\mathbf{x}_{i}$, utilizing a neural network. More specifically, DBU models predict \textit{pseudo-counts} (i.e., class counts of the pseudo-observations) $\alpha
^{(c)}_{\data} = \sum_{j=1}^{N} \mathbbm{1}_{\{\tilde{y}^{(j)}_{j} = c\}}$ for each class $\forall c \in [C]$. Then, the concentration parameters of the posterior Dirichlet distribution can be obtained in a closed form, leveraging the conjugacy of the Dirichlet and Categorical distributions. In summary, the DBU model can be expressed as a Bayesian model as follows: 
\begin{align*}
\textrm{\textbf{Prior}} \quad \qquad \qquad \boldsymbol{\pi} &\sim \Dir(\boldsymbol{\alpha}_{\prior}(\mathbf{x}_{i})), \\
\textrm{\textbf{Likelihood}} \quad \{\tilde{y}^{(j)}_{i}\}_{j=1}^{N}| \boldsymbol{\pi} &\sim \Cat(\boldsymbol{\pi}). \\
\textrm{\textbf{Posterior}} \quad \boldsymbol{\pi}|\{\tilde{y}^{(j)}_{i}\}_{j=1}^{N} &\sim \Dir(\boldsymbol{\alpha}_{\post}(\mathbf{x}_{i})), \quad \boldsymbol{\alpha}_{\post}(\mathbf{x}_{i}) = \boldsymbol{\alpha}_{\prior}(\mathbf{x}_{i}) + \boldsymbol{\alpha}_{\data}(\mathbf{x}_{i}),
\end{align*}
where $\boldsymbol{\pi}$ is a class probability and the pseudo-counts are computed as $\alpha^{(c)}_{\data}(\mathbf{x}_{i}) = \sum_{j=1}^{N} \mathbbm{1}_{\{\tilde{y}^{(j)}_{j} = c\}}, \forall c \in [C]$. 
The core aspect of the DBU model involves predicting pseudo-counts using a neural network and computing the posterior Dirichlet distribution. Specifically, the pseudo-counts are obtained as follows:
\begin{align*}
\boldsymbol{\alpha}_{\data}(\mathbf{x}_{i}) = h(g_{\boldsymbol{\phi}}(f_{\boldsymbol{\theta}}(\mathbf{x}_{i})),
\end{align*}
where $f_{\boldsymbol{\theta}} : \mathbb{R}^{D} \rightarrow \mathbb{R}^{H}$, $g_{\boldsymbol{\phi}} : \mathbb{R}^{H} \rightarrow \mathbb{R}^{C}$, $\boldsymbol{\theta}$, $\boldsymbol{\phi}$, and $h$ are the feature extractor, classifier, parameters of the feature extractor, parameters of the classifier, and activation function, respectively. With the conventional choice for the concentration parameters of the prior Dirichlet distribution ($\boldsymbol{\alpha}_{\prior}(\mathbf{x}_{i}) = \mathbf{1}$), the concentration parameters of the posterior Dirichlet distribution of the DBU models can be expressed as follows: 
\begin{align*}
\boldsymbol{\alpha}_{\post}(\mathbf{x}_{i}) = \mathbf{1} + h(g_{\boldsymbol{\phi}}(f_{\boldsymbol{\theta}}(\mathbf{x}_{i}))).
\end{align*}
Representative DBU models, including PostNet \cite{charpentier2020posterior}, NatPN \cite{charpentier2021natural}, EDL \cite{sensoy2018evidential}, and $\mathcal{I}$-EDL \cite{deng2023uncertainty}, can be interpreted within the framework explained above. The comparison of the DBU models under the input-dependent Dirichlet-Categorical model framework is detailed in \cref{Parameter Option Comparison}. \begin{table*}[hbtp]
\caption{The comparison between typical DBU models \cite{sensoy2018evidential, charpentier2020posterior, charpentier2021natural, deng2023uncertainty} and DAEDL in the Bayesian context. $N_{c}$ denotes the number of observations for each class \cite{charpentier2020posterior}, and $N_{H}$ is a hyperparameter that corresponds to the certainty budget \cite{charpentier2021natural}. $\boldsymbol{\beta}$ denotes the parameter of the density estimator.}
\label{Parameter Option Comparison}
\begin{center}
\begin{sc}
\resizebox{\columnwidth}{!}{%
\begin{tabular}{@{}llll@{}}
\toprule
 & \textbf{Prior} & \textbf{Pseudo-Counts} & \textbf{Posterior} \\ \midrule
PostNet & $\alpha^{(c)}_{\prior} = 1$  & $\alpha^{(c)}_{\data} =  N_{c} \times  p(\mathbf{z}^{\star};\boldsymbol{\beta}^{(c)})$  & $\alpha^{(c)}_{\post} = 1 + N_{c} \times p(\mathbf{z}^{\star};\boldsymbol{\beta}^{(c)})$  \\
NatPN & $\alpha^{(c)}_{\prior} = 1$  & $\alpha^{(c)}_{\data} =  N_{H} \times  p(\mathbf{z}^{\star};\boldsymbol{\beta}) \times (g_{\boldsymbol{\phi}}(\mathbf{z}^{\star}))_{c}$  & $\alpha^{(c)}_{\post}= 1 + N_{H} \times p(\mathbf{z}^{\star};\boldsymbol{\beta}) \times  (g_{\boldsymbol{\phi}}(\mathbf{z}^{\star}))_{c}$  \\
EDL & $\alpha^{(c)}_{\prior} = 1$ & $\alpha^{(c)}_{\data} = \ReLU((g_{\boldsymbol{\phi}}(\mathbf{z}^{\star}))_{c})$ & $\alpha^{(c)}_{\post} = 1 + \ReLU((g_{\boldsymbol{\phi}}(\mathbf{z}^{\star}))_{c})$  \\
$\mathcal{I}$-EDL & $\alpha^{(c)}_{\prior} = 1$   &$\alpha^{(c)}_{\data} = \Softplus((g_{\boldsymbol{\phi}}(\mathbf{z}^{\star}))_{c})$  & $\alpha^{(c)}_{\post} = 1 + \Softplus((g_{\boldsymbol{\phi}}(\mathbf{z}^{\star}))_{c})$  \\ \midrule
DAEDL (Training) & $\alpha^{(c)}_{\prior} = 0$ &$\alpha^{(c)}_{\data} = \exp \left((g_{\boldsymbol{\phi}}(\mathbf{z}^{\star}))_{c}\right)$  & $\alpha^{(c)}_{\post} = \exp \left((g_{\boldsymbol{\phi}}(\mathbf{z}^{\star}))_{c}\right)$  \\
DAEDL (Prediction) & $\alpha^{(c)}_{\prior} = 0$ &$\alpha^{(c)}_{\data} = \exp \left((g_{\boldsymbol{\phi}}(\mathbf{z}^{\star}))_{c} \times s(\mathbf{z}^{\star}) \right)$  & $\alpha^{(c)}_{\post} = \exp \left((g_{\boldsymbol{\phi}}(\mathbf{z}^{\star}))_{c} \times s(\mathbf{z}^{\star};\boldsymbol{\beta})) \right)$  \\ \bottomrule
\end{tabular}
}
\end{sc}
\end{center}
\end{table*}

\paragraph{Limitation of conventional parameterization of DBU models.}
\label{problem of conventional parameterization}
For a typical Dirichlet-Categorical model in Bayesian statistics, a commonly employed non-informative prior is a uniform Dirichlet distribution (i.e., $\Dir(\mathbf{1})$). Thus, it is natural for models like PostNet and NatPN to adopt this prior. Moreover, EDL models \cite{sensoy2018evidential, deng2023uncertainty} can also be interpreted as using a uniform Dirichlet prior implicitly. \cref{DBU Bayesian} establishes that typical DBU models can be interpreted as using $\boldsymbol{\pi} \sim \Dir(\mathbf{1})$ under this framework. However, this seemingly reasonable choice of the concentration parameters of the prior Dirichlet distribution in the DBU models poses inherent challenges. Notably, the mechanism for obtaining the posterior distribution differs from typical Bayesian models. Obtaining the parameters of the posterior in the DBU models involves pseudo-counts, which are estimated by a neural network, rather than actual class counts. From \cref{Parameter Option Comparison}, we can see that the range of the pseudo-count is $[0,\infty)$. Consequently, several failure cases arise due to the unconstrained range of the pseudo-counts. Below, we illustrate some representative failure cases.

\begin{theorem}
\label{DBU Bayesian}
\textbf{\textit{(Bayesian Interpretation of DBU models)}}
In the Bayesian context, typical DBU models can be interpreted as predicting an input-dependent posterior distribution of the Dirichlet-Categorical model with a uniform prior $\boldsymbol{\pi} \sim \Dir(\mathbf{1})$. 
\end{theorem}

\underline{(i) $C=3, \ \ \boldsymbol{\alpha}_{\data}(\mathbf{x}_{1}^{\star}) = [0.1,0.01,0.01]$}

Consider the scenario where the magnitude of the concentration parameters of the prior Dirichlet distribution dominates the pseudo-counts in the classification task with $C=3$. Suppose that $\mathbf{x}_{1}^{\star}$ is a highly likely ID testing example that corresponds to the first class, with the pseudo-counts computed as $\boldsymbol{\alpha}_{\data}(\mathbf{x}_{1}^{\star}) = [0.1,0.01,0.01]$. Then, the concentration parameters of the posterior are estimated as $\boldsymbol{\alpha}_{\post}(\mathbf{x}_{1}^{\star}) = [1.1, 1.01, 1.01]$. However, the maximum expected 
class probability is computed as ${1.1}/{3.12}$, yielding a value close to ${1}/{3}$. Then, the highly likely ID data point will be misclassified as OOD. 

\underline{ii) $C=10, \quad \boldsymbol{\alpha}_{\data}(\mathbf{x}_{2}^{\star}) = [10,0,0,\cdots,0]$} 

Consider the scenario where pseudo-counts exist only for one class in the classification task with $C=10$. Suppose that $\mathbf{x}_{2}^{\star}$ is a highly likely ID testing example that corresponds to the first class, with the pseudo-counts computed as $\boldsymbol{\alpha}_{\data}(\mathbf{x}_{2}^{\star}) = [10,0,0,\cdots,0]$. Then, the concentration parameters of the posterior is estimated as $\boldsymbol{\alpha}_{\post}(\mathbf{x}_{2}) = [11,1,1,\cdots,1]$. Subsequently, the maximum expected class probability is computed as $11/20 = 0.55$. This result is counter-intuitive, as typical classification models (e.g., softmax) usually output the maximum class probability close to 1 for a highly ID data point like $\mathbf{x}_{2}^{\star}$. 

These counter-intuitive results in the expected class probability occur due to the inherent challenge of achieving the appropriate balance in the magnitude between the concentration parameters of the prior Dirichlet distribution and the pseudo-counts. To overcome this problem, it is necessary to employ the optimal concentration parameters of the prior Dirichlet distribution that maintains the right balance with the pseudo-counts. In the absence of specific information, it is reasonable to set the concentration parameter of the prior Dirichlet distribution for each class to be the same (i.e., $\boldsymbol{\alpha}_{\prior} = \alpha \mathbf{1}$), where $\alpha$ represents the magnitude and $\mathbf{1} \in \mathbb{R}^{C}$. However, determining the optimal magnitude $\alpha$ is non-trivial, because of the unconstrained range of the magnitude of the pseudo-counts. While an optimal value of $\alpha$ in the DBU models might exist in some contexts, $\alpha = 1$ is applied in most tasks. However, as illustrated above, numerous counter-intuitive results may occur in the expected class probability attributed to setting $\alpha = 1$. We hypothesize that prior misspecification in DBU models can be a significant factor that contributes to the limited classification performance of DBU models. DAEDL successfully addresses this limitation by employing an alternative prior that eliminates the need for this nonessential balancing.

\paragraph{Bayesian interpretation of DAEDL.}
We interpret DAEDL within the input-dependent Dirichlet-Categorical model framework to provide insights into how DAEDL resolves the limitation in the conventional parameterization of DBU models. As established in \cref{Bayesian Interpretation}, DAEDL can be interpreted as predicting the input-dependent posterior distribution of the Dirichlet-Categorical model with an improper prior $\boldsymbol{\pi} \sim \Dir(\mathbf{0})$. DAEDL essentially eliminates the challenge of determining the optimal parameter value of the prior by setting it to zero.  
Although the improper prior is not a valid probability distribution, the posterior Dirichlet distribution is valid if all pseudo-counts are positive, which holds in practice for neural network models. The improper prior exhibits favorable theoretical properties, such as being a uniform prior in the log-scale of the parameter (i.e., $\log \boldsymbol{\pi}$) \cite{gelman1995bayesian}.

To summarize, while typical DBU models can be interpreted as using a uniform prior $\boldsymbol{\pi} \sim \Dir(\mathbf{1})$ in the Bayesian context, DAEDL can be viewed as employing an improper prior $\boldsymbol{\pi} \sim \Dir(\mathbf{0})$. While opting for a uniform prior is standard in typical models, it poses challenges within the context of DBU models, due to the challenge of achieving an appropriate balance between the parameter of the prior and pseudo-counts. This challenge may lead to counter-intuitive results in the expected class probabilities, which could be a potential reason for the limited performance of DBU models in classification tasks. Therefore, within the DBU model framework, we assert that employing an improper prior $\Dir(\mathbf{0})$ is a more practical choice. 

\subsection{Additional Explanation about \cref{Relationship with AdaTS}}
\label{Appendix : Additional Explanation about thm1}
We provide an additional insight into \cref{Relationship with AdaTS}. To begin, we define the EDL equipped with the proposed parameterization as EDL\textsuperscript{+} for notational convenience. Despite the existence of other models related to adaptive temperature scaling, here we focus on Adaptive Temperature Scaling (AdaTS) \cite{joy2023sample}, which serves as a representative model.

We first clarify the relationship between the four models: i) Softmax, ii) EDL\textsuperscript{+}, iii) AdaTS, and iv) DAEDL. \cref{Fig3} illustrates the relationship of these models graphically. The orange arrow signifies the enhancement in the model calibration performance achieved by adaptive temperature scaling. The blue arrow denotes the improvement in the uncertainty estimation performance of the model achieved by adopting the EDL structure. As demonstrated in \cref{Fig3}, DAEDL can be interpreted in two manners: i) calibrated version of EDL\textsuperscript{+}, and ii) AdaTS with uncertainty estimation ability. First, DAEDL can be interpreted as the calibrated version of EDL\textsuperscript{+}, accomplished by dividing the logits by the sample-dependent temperature before applying the exponential activation function. Given the demonstrated effectiveness of \textit{adaptive temperature scaling} in improving model calibration and OOD detection \cite{balanya2022adaptive, joy2023sample, krumpl2024ats}, we can expect that DAEDL will exhibit similar benefits and improvements over EDL\textsuperscript{+}. Second, DAEDL can be interpreted as AdaTS with uncertainty estimation capability. Given the superior uncertainty estimation performance of the EDL\textsuperscript{+} compared to the softmax model, DAEDL is expected to exhibit improved uncertainty estimation performance over AdaTS. In summary, DAEDL can be interpreted as a model created by integrating enhancements from two distinct directions to the softmax model, to improve both calibration and uncertainty estimation performance. 
\begin{figure}[hbtp]
\includegraphics[scale = 0.38]{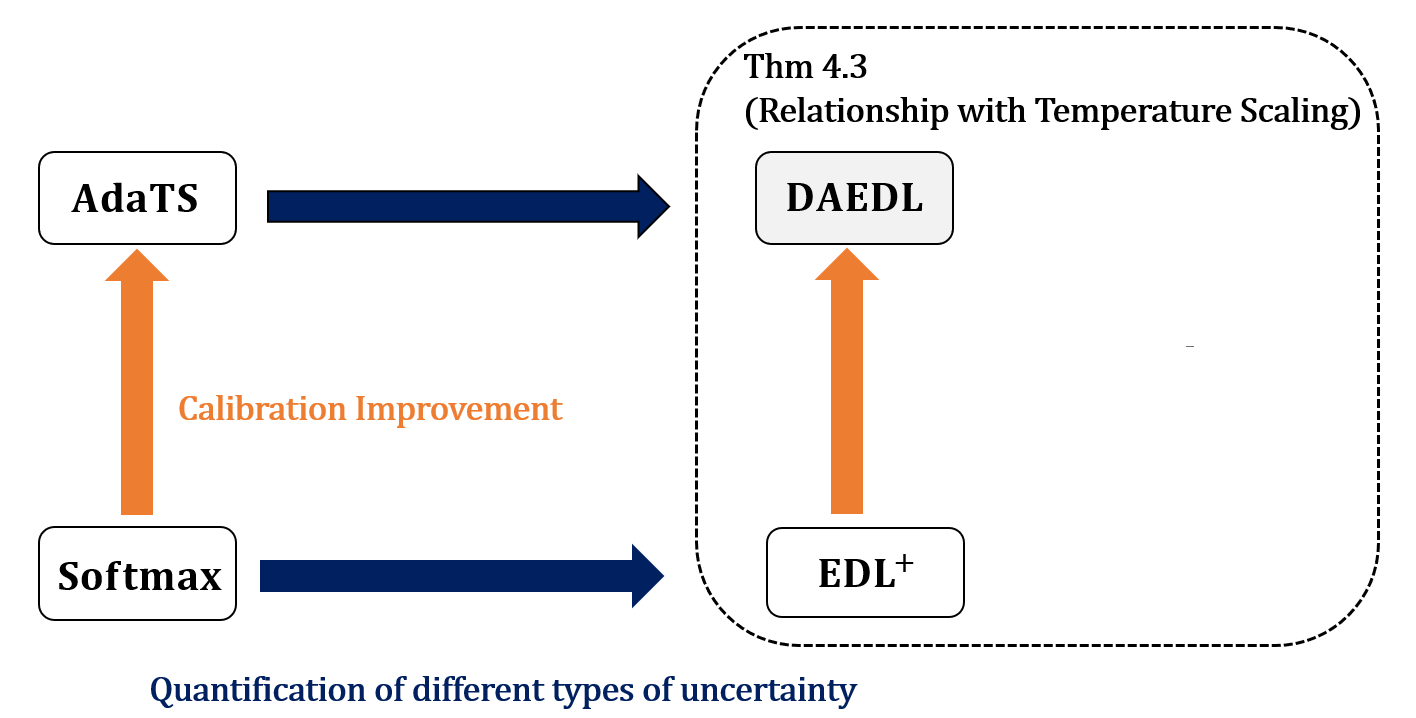}
\caption{Relationship between Softmax, EDL\textsuperscript{+}, AdaTS, and DAEDL}
\label{Fig3}
\end{figure}

\newpage

\section{Algorithm}
\label{Appendix : Algorithm}
We present the algorithms for DAEDL. Specifically, the algorithms for training, density estimation, and prediction are provided in \cref{alg:train}, \cref{alg:density}, and \cref{alg:prediction}, respectively. These algorithms are intuitive and easy to implement.
\begin{algorithm}[hbtp]
   \caption{DAEDL Training}
   \label{alg:train}
\begin{algorithmic}
   \STATE {\bfseries Input:} Training data $\mathcal{D}_{tr} = \{(\mathbf{x}_{i},y_{i})\}_{i=1}^{N}$, initial model parameters $\{\boldsymbol{\theta}, \boldsymbol{\phi}\}$, maximum epoch $M$, learning rate $\eta$, batch size $B$, regularization parameter $\lambda$
   \FOR{$i=1$ {\bfseries to} $M$}
   \STATE SGD update $\bigl\{ \boldsymbol{\theta} = \{W^{(l)}, b^{(l)}\}_{l=1}^{L} , \ \boldsymbol{\phi} \bigl\}$.
   \STATE Apply spectral normalization to $ \{W^{(l)}\}_{l=1}^{L}$.
   \ENDFOR
   \STATE {\bfseries Output:} Trained model parameter $\{\hat{\boldsymbol{\theta}}, \hat{\boldsymbol{\phi}}\}$
\end{algorithmic}
\end{algorithm}
\begin{algorithm}[hbtp]
   \caption{DAEDL Density Estimation}
   \label{alg:density}
\begin{algorithmic}
   \STATE {\bfseries Input:} Training data $\mathcal{D}_{tr} = \{(\mathbf{x}_{i},y_{i})\}_{i=1}^{N}$, trained feature extractor parameter $\hat{\boldsymbol{\theta}}$
   \FOR{$c=1$ {\bfseries to} $C$}
   \STATE $\hat{\omega}_{c} \leftarrow \frac{N_{c}}{N}$, \ \ $N_{c} = \sum_{i=1}^{N} \mathbbm{1}_{\{y_{i} = c\}}$
   \STATE $\hat{\boldsymbol{\mu}}_{c} \leftarrow \frac{1}{N_{c}} \sum_{\{i:y_{i} = c\}}^{} f_{\hat{\boldsymbol{\theta}}}(\mathbf{x}_{i})$
   \STATE $\hat{\Sigma}_{c} \leftarrow \frac{1}{N_{c}-1} \sum_{\{i:y_{i} = c\}}^{} (f_{\hat{\boldsymbol{\theta}}}(\mathbf{x}_{i}) - \hat{\boldsymbol{\mu}}_{c}) (f_{\hat{\boldsymbol{\theta}}}(\mathbf{x}_{i}) - \hat{\boldsymbol{\mu}}_{c})^{T}$
   \ENDFOR
   \STATE {\bfseries Output:} GDA parameter $\hat{\boldsymbol{\omega}} = \{\hat{\omega}_{c}\}_{c=1}^{C}, \boldsymbol{\hat{\mu}} = \{\hat{\boldsymbol{\mu}}_{c}\}_{c=1}^{C}, \hat{\boldsymbol{\Sigma}} = \{\hat{\Sigma}_{c}\}_{c=1}^{C}$
\end{algorithmic}
\end{algorithm}
\begin{algorithm}[hbtp]
   \caption{DAEDL Prediction}
   \label{alg:prediction}
\begin{algorithmic}
   \STATE {\bfseries Input:} Testing example $\mathbf{x}^{\star}$, trained model parameter $\{\hat{\boldsymbol{\theta}}, \hat{\boldsymbol{\phi}}\}$, GDA parameter $\boldsymbol{\hat{\omega}}, \ \boldsymbol{\hat{\mu}},\ \boldsymbol{\hat{\Sigma}}$
   \STATE Estimate the log feature space density: $\log p(\mathbf{z}^{\star} = f_{\boldsymbol{\theta}}(\mathbf{x}^{\star})|\hat{\boldsymbol{\omega}}, \hat{\boldsymbol{\mu}}, \hat{\Sigma}) $
   \STATE Compute the normalized feature space density: $s(\mathbf{x}^{\star}) = \Clip \left( \frac{\log p(\mathbf{z}^{\star} = f_{\hat{\boldsymbol{\theta}}}(\mathbf{x}^{\star}))-d_{\min}}{d_{\max} - d_{\min}} \right)$ 
   \STATE Compute the concentration parameters: $\boldsymbol{\alpha}_{\hat{\boldsymbol{\theta}}, \hat{\boldsymbol{\phi}}}(\mathbf{x}^{\star}) = \exp \left(g_{\hat{\boldsymbol{\phi}}}(f_{\hat{\boldsymbol{\theta}}}(\mathbf{x}^{\star})) \times s(\mathbf{x}^{\star})\right)$
   \STATE {\bfseries Output:} Concentration parameters $\boldsymbol{\alpha}_{\hat{\boldsymbol{\theta}}, \hat{\boldsymbol{\phi}}}(\mathbf{x}^{\star})$
\end{algorithmic}
\end{algorithm}

\section{Experimental Details}
\label{Appendix : Experimental Details}
In this section, we provide the details about our experiments. First, we provide a detailed description of the datasets employed in our experiments. Second, we provide the implementation details. 

\subsection{Datasets}
\label{Appendix : Experimental Details_Datasets}
Following \citet{charpentier2020posterior} and \citet{deng2023uncertainty}, we mainly use two image classification datasets : (i) {MNIST} \cite{lecun1998mnist} and (ii) {CIFAR-10} \cite{krizhevsky2009learning} as our ID dataset. For MNIST, {Kuzushiji-MNIST (KMNIST)} \cite{clanuwat2018deep} and {FashionMNIST (FMNIST)} \cite{xiao2017fashion} serve as OOD datasets. Regarding CIFAR-10, {Street View House Numbers (SVHN)} \cite{netzer2011reading} and {CIFAR-100} \cite{krizhevsky2009learning} were employed as OOD datasets. Furthermore, to assess our model's capability for detecting distribution shifts, we utilize {MNIST-C} \cite{mu2019mnist} and {CIFAR-10-C} \cite{hendrycks2019benchmarking}, which are created by introducing distribution shifts (corruption) to MNIST and CIFAR-10, respectively. A detailed description of each dataset is provided below. 

\paragraph{MNIST} \cite{lecun1998mnist} is a dataset comprising grayscale images of handwritten digits (0 through 9). MNIST is widely used as a benchmark dataset for evaluating machine learning algorithms. It comprises 60,000 training images and 10,000 testing images. Each image is represented as a $1 \times 28 \times 28$ tensor. We partitioned the training samples into a training set and a validation set with a ratio of $0.8:0.2$.
\paragraph{Kuzushiji-MNIST (KMNIST)} \cite{clanuwat2018deep} serves as a more challenging alternative of the MNIST dataset. Similar to MNIST, it comprises grayscale images represented by $1 \times 28 \times 28$ tensor and includes 60,000 training images and 10,000 testing images. However, KMNIST features a handwritten Japanese Hiragana script, originally designed for researching the recognition of historical Japanese characters. In our experiments, we employed KMNIST as an OOD dataset for MNIST. 

\paragraph{FashionMNIST (FMNIST)} \cite{xiao2017fashion} serves as another alternative to the MNIST dataset. Similar to MNIST and KMNIST, it comprises grayscale images represented by $1 \times 28 \times 28$ tensor and includes 60,000 training images and 10,000 testing images. However, the FMNIST dataset features images of various fashion items. Specifically, it consists of 10 classes, each representing specific fashion items such as T-shirt, trouser, pullover, and others. In our experiments, we employed FMNIST as an OOD dataset for MNIST.

\paragraph{CIFAR-10} \cite{krizhevsky2009learning} is a dataset comprising color images of different animals and objects. CIFAR-10 is widely used as a benchmark dataset for evaluating machine learning algorithms. It comprises of 50,000 training images and 10,000 testing images. The dataset is divided into ten classes: airplane, automobile, bird, cat, deer, dog, frog, horse, ship, and truck. Each image in the dataset is represented as $3 \times 32 \times 32$ tensor. We partitioned the training samples into a training set and a validation set with a ratio of $0.95:0.05$.

\paragraph{Street View House Numbers (SVHN)} \cite{netzer2011reading} is a dataset consisting of cropped images of house numbers from Google Street View. The dataset is divided into 10 classes, 1 for each digit. There are 73,257 training images, 26,032 testing images, and 531,131 additional images. In our experiments, we employed SVHN as an OOD dataset for CIFAR-10.

\paragraph{CIFAR-100} \cite{krizhevsky2009learning} is an extension of the CIFAR-10 dataset. It comprises 50,000 training images and 10,000 testing images, divided among 100 different classes. Each image in this dataset is represented as a $3 \times 32 \times 32$ tensor. In our experiments, we employed CIFAR-100 as an OOD dataset for CIFAR-10.

\paragraph{MNIST-C} \cite{mu2019mnist} is a dataset created by applying 15 different types of corruptions (distribution shift) to the MNIST test set. These corruptions include shot noise, impulse noise, glass blur, motion blur, shear, scale, rotate, brightness, translate, stripe, fog, spatter, dotted line, zigzag, and canny edges. In our experiments, we utilized a static MNIST-C dataset, where severity levels are fixed \cite{mu_2019_3239543}, as an OOD dataset for the distribution shift detection task.

\paragraph{CIFAR-10-C} \cite{hendrycks2019benchmarking} is a dataset generated by applying 19 different types of corruption (distribution shift) to the CIFAR-10 test set. These corruptions include Gaussian noise, shot noise, impulse noise, defocus blur, glass blur, motion blur, zoom blur, snow, frost, fog, brightness, contrast, elastic transform, pixelate, jpeg compression, speckle noise, Gaussian blur, splatter, and saturate. Each corruption is applied at five different severity levels $\mathcal{C} \in \{1,2,3,4,5\}$. In our experiments, we employed CIFAR-10-C as an OOD dataset for the distribution shift detection task. In essence, conducting distribution shift detection for CIFAR-10-C is akin to performing OOD detection with 95 (19 $\times$ 5) different OOD datasets. 

\subsection{Implementation Details}
\label{Appendix : Experimental Details_ Implementation Details}
For a fair comparison, we followed \citet{charpentier2020posterior} and  \citet{deng2023uncertainty} regarding the choice of the backbone. Specifically, we implemented a configuration of 3 convolutional layers and 3 dense layers ({ConvNet}) when utilizing MNIST as the ID dataset. When CIFAR-10 served as the ID dataset, we used {VGG-16}. For VGG-16, dropout with a rate of $0.5$ was applied. FMNIST and KMNIST were used as OOD datasets for MNIST, while SVHN and CIFAR-100 were employed as OOD datasets for CIFAR-10. To prevent overfitting, early stopping based on the validation loss was implemented for both datasets. In the case of MNIST, training extended up to 50 epochs with a batch size of 64, and for CIFAR-10, we trained up to 100 epochs with the same batch size. The Adam optimizer and LambdaLR scheduler were employed for both datasets. The learning rate ($\eta$) and regularization parameter ($\lambda$) were determined through a grid search, yielding the optimal values of 
$(\eta, \lambda) = (10^{-3}, 5 \times 10^{-2})$. Notably, DAEDL was robust to the hyperparameter choice and only required minimal tuning. A summary of the implementation details is presented in \cref{Implementation Details}.
\begin{table*}[hbtp]
\caption{Implementation details of our experiments. $B$, $p_{\drop}$, $\eta$, $lr_{\lambda}$, $\lambda$, and $T_{\max}$ denote the batch size, dropout rate, learning rate, parameter of the scheduler, regularization parameter, and maximum epoch,  respectively. }
\label{Implementation Details}
\vskip 0.15in
\begin{center}
\begin{small}
\begin{sc}
\begin{tabular}{@{}lllllllllll@{}}
\toprule
ID Dataset & Backbone & Optimizer & Scheduler & $B$ & $p_{\drop}$ & $\eta$ & $lr_{\lambda}$ & $\lambda$ & $T_{\max}$ \\ \midrule
MNIST  & ConvNet & Adam & LambdaLR & 64 & \ \ - & $10^{-3}$ & $0.95^{\epochs}$ & $5 \times 10^{-2}$ & 50             \\
CIFAR-10 & VGG-16 & Adam & LambdaLR & 64 & 0.5 & $10^{-3}$
& $0.95^{\epochs}$ & $5 \times 10^{-2}$ & 100 \\ \bottomrule
\end{tabular}
\end{sc}
\end{small}
\end{center}
\vskip -0.15in
\end{table*}

\section{Additional Experimental Results}
\subsection{Additional Results in OOD Detection (\cref{Exp : OOD Detection})}
\label{Appendix : OOD Detection}
\cref{OOD_DETECTION_AUROC} presents the area under the receiver operating characteristic curve (AUROC) scores measured for the OOD detection performance. We can observe that DAEDL outperformed the competitors in all tasks. 
\begin{table*}[hbtp]
\centering
\caption{AUROC scores of OOD detection based on aleatoric and epistemic uncertainty. The results of EDL and $\mathcal{I}$-EDL were obtained from \citet{deng2023uncertainty}.}
\label{OOD_DETECTION_AUROC}
\vskip 0.15in
\begin{small}
\begin{sc}
\resizebox{\columnwidth}{!}{%
\begin{tabular}{@{}lcccccccc@{}}
\toprule
 &
  \multicolumn{2}{c}{\textbf{MNIST $\rightarrow$ KMNIST}} &
  \multicolumn{2}{c}{\textbf{MNIST $\rightarrow$ FMNIST}} & 
  \multicolumn{2}{c}{\textbf{CIFAR-10 $\rightarrow$ SVHN}} &
  \multicolumn{2}{c}{\textbf{CIFAR-10 $\rightarrow$ CIFAR-100}} \\ \midrule
& \textbf{Alea.}   & \textbf{Epis.}   & \textbf{Alea.}    & \textbf{Epis.}   & \textbf{Alea.}    & \textbf{Epis.}   & \textbf{Alea.}   & \textbf{Epis.}  \\ \midrule
EDL & 96.59 $\pm$ 0.6 & 96.18 $\pm$ 1.3 & 96.49 $\pm$ 0.8 & 96.22 $\pm$ 1.3 & 80.64 $\pm$ 4.2 & 81.06 $\pm$ 4.5 & 80.96 $\pm$ 0.8 & 80.63 $\pm$ 1.0  \\
$\mathcal{I}$-EDL & 98.00 $\pm$ 0.3 & 97.97 $\pm$ 0.3 &  97.99 $\pm$ 0.3 & 97.97 $\pm$ 0.3 &   87.58 $\pm$ 2.0 & 86.79 $\pm$ 1.3 &  83.55 $\pm$ 0.7 & 82.15 $\pm$ 0.5 \\ \midrule
\textbf{DAEDL} & \textbf{99.88 $\pm$ 0.0} & \textbf{99.90 $\pm$ 0.0}  & \textbf{99.77 $\pm$ 0.1}  & \textbf{99.82 $\pm$ 0.0}  & \textbf{89.10 $\pm$ 1.0} & \textbf{89.24 $\pm$ 1.0} & \textbf{85.94 $\pm$ 0.1} & \textbf{86.04 $\pm$ 0.1}  \\ \bottomrule
\end{tabular}
}
\end{sc}
\end{small}
\vskip -0.15in
\end{table*}

\subsection{Additional Results in Distribution Shift Detection (\cref{Exp : Distribution Shift Detection})}
\label{Appendix : Distribution Shift Detection}
\paragraph{MNIST $\rightarrow $ MNIST-C} \cref{Corrupted Data Detection - MNIST} demonstrates the results for the distribution shift detection task on the MNIST-C dataset. The results with two different score metrics (AUPR and AUROC) and uncertainty measures (aleatoric and epistemic) are provided. We can observe that DAEDL outperforms the competitors by a significant margin regardless of the score metric and uncertainty measures. 

\begin{table*}[hbtp]
\caption{AUPR and AUROC scores of distribution shift detection based on aleatoric and epistemic uncertainty estimates for the MNIST-C dataset. Alea.AUPR indicates that aleatoric uncertainty is employed as an uncertainty measure while AUPR is utilized as a performance metric. Similarly, Alea.AUROC denotes the results with aleatoric uncertainty and AUROC. In addition, EPIS.AUPR and EPIS.AUROC denote the results obtained by employing epistemic uncertainty while utilizing AUPR and AUROC as a performance metric, respectively.}
\vskip 0.15in
\label{Corrupted Data Detection - MNIST}
\begin{center}
\begin{small}
\begin{sc}
\begin{tabular}{@{}lllll@{}}
\toprule
& \textbf{Alea. AUPR} & \textbf{Epis. AUPR} & \textbf{Alea. AUROC} & \textbf{Epis. AUROC}\\ \midrule
MSP& 78.54 $\pm$ 0.3 & \qquad - & 79.99 $\pm$ 0.4 & \qquad - \\
EDL   & 82.75 $\pm$ 0.8 & 82.75 $\pm$ 0.8 & 82.59 $\pm$ 0.6 & 82.59  $\pm$ 0.6 \\
$\mathcal{I}$-EDL & 86.06 $\pm$ 0.5  & 86.04 $\pm$ 0.5   & 85.83 $\pm$ 0.7  & 85.80 $\pm$ 0.7  \\ \midrule 
\textbf{DAEDL} & \textbf{92.43 $\pm$ 0.3} & \textbf{92.51 $\pm$ 0.3} & \textbf{91.84 $\pm$ 0.3} & \textbf{91.99 $\pm$ 0.3} \\ \bottomrule
\end{tabular}
\end{sc}
\end{small}
\end{center}
\vskip -0.15in
\end{table*}

\paragraph{CIFAR-10 $\rightarrow$ CIFAR-10-C} 
\cref{CIFAR-10-C_AVG} demonstrates four sub-figures that represent the results of the distribution shift detection task. The four sub-figures differ by the uncertainty measure (aleatoric and epistemic) and the performance metric (AUPR and AUROC). The results indicate that DAEDL consistently outperforms the competitors regardless of the score and uncertainty metrics. Notably, the performance gap between DAEDL and the competitors is higher when aleatoric uncertainty was applied as an uncertainty measure. As aleatoric uncertainty corresponds to the inherent uncertainty within the data, it is a more suitable measure for capturing the distribution shift that occurred in the data. Therefore, the results align with our expectations. 

\begin{figure}[hbtp]
    \centering
    \begin{subfigure}{}
        \centering
        \includegraphics[scale = 0.5]{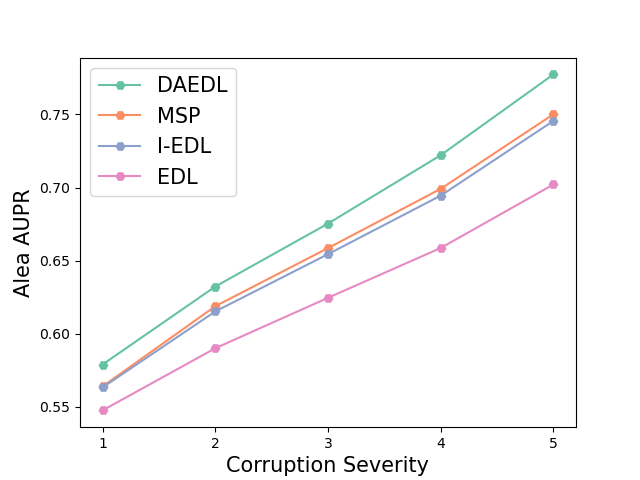}
        \label{Alea AUPR}
    \end{subfigure}
    \hfill
    \begin{subfigure}{}
        \centering
        \includegraphics[scale = 0.5]{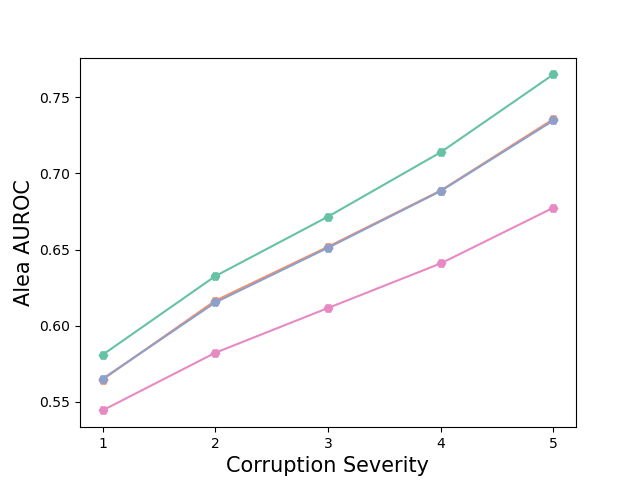}
        \label{Alea AUROC}
    \end{subfigure}

    \medskip

    \begin{subfigure}{}
        \centering
        \includegraphics[scale = 0.5]{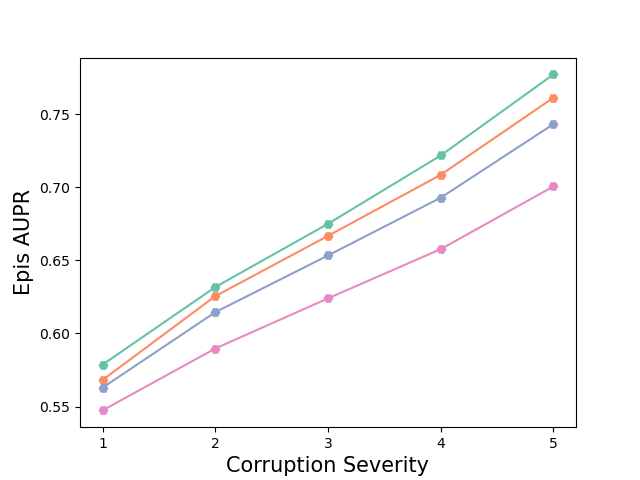}
        \label{Epis AUPR}
    \end{subfigure}
    \hfill
    \begin{subfigure}{}
        \centering
        \includegraphics[scale = 0.5]{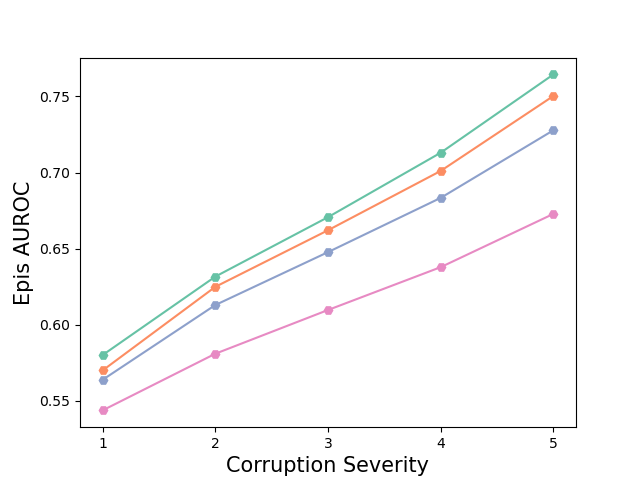}
        \label{Epis AUROC}
    \end{subfigure}

    \caption{The plot shows the average scores for distribution shift detection in the CIFAR-10-C dataset. The scores were obtained by averaging over 5 independent runs. The score for each run was obtained by averaging over 19 different corruptions. The left-side figures show the results using AUPR, while the right-side figures present the outcomes with AUROC. Additionally, the top row displays the results obtained using aleatoric uncertainty, and the bottom row features the results obtained using epistemic uncertainty.}
    \label{CIFAR-10-C_AVG}
\end{figure}

\paragraph{Results for various corruption types on CIFAR-10 $\rightarrow$ CIFAR-10-C} \cref{Alea_AUPR_COMBINED}, \cref{Alea_AUROC_COMBINED}, \cref{Epis_AUPR_Combined}, and \cref{Epis_AUROC_Combined} demonstrate the graphical illustrations of the results for 19 different corruptions. The four figures differ by the uncertainty measure (aleatoric and epistemic) and the performance metric (AUPR and AUROC). In each figure, 19 sub-figures differ by the type of corruption applied. The list of corruptions includes Gaussian noise, shot noise, impulse noise, defocus blur, glass blur, motion blur, zoom blur, snow, frost, fog, brightness, contrast, elastic transform, pixelate, jpeg compression, speckle noise, Gaussian blur, splatter, and saturate. From the figures, we can observe that DAEDL outperformed the competitors in most of the corruption types and severity levels. 

\begin{figure}[hbtp]
\includegraphics[scale = 0.22]{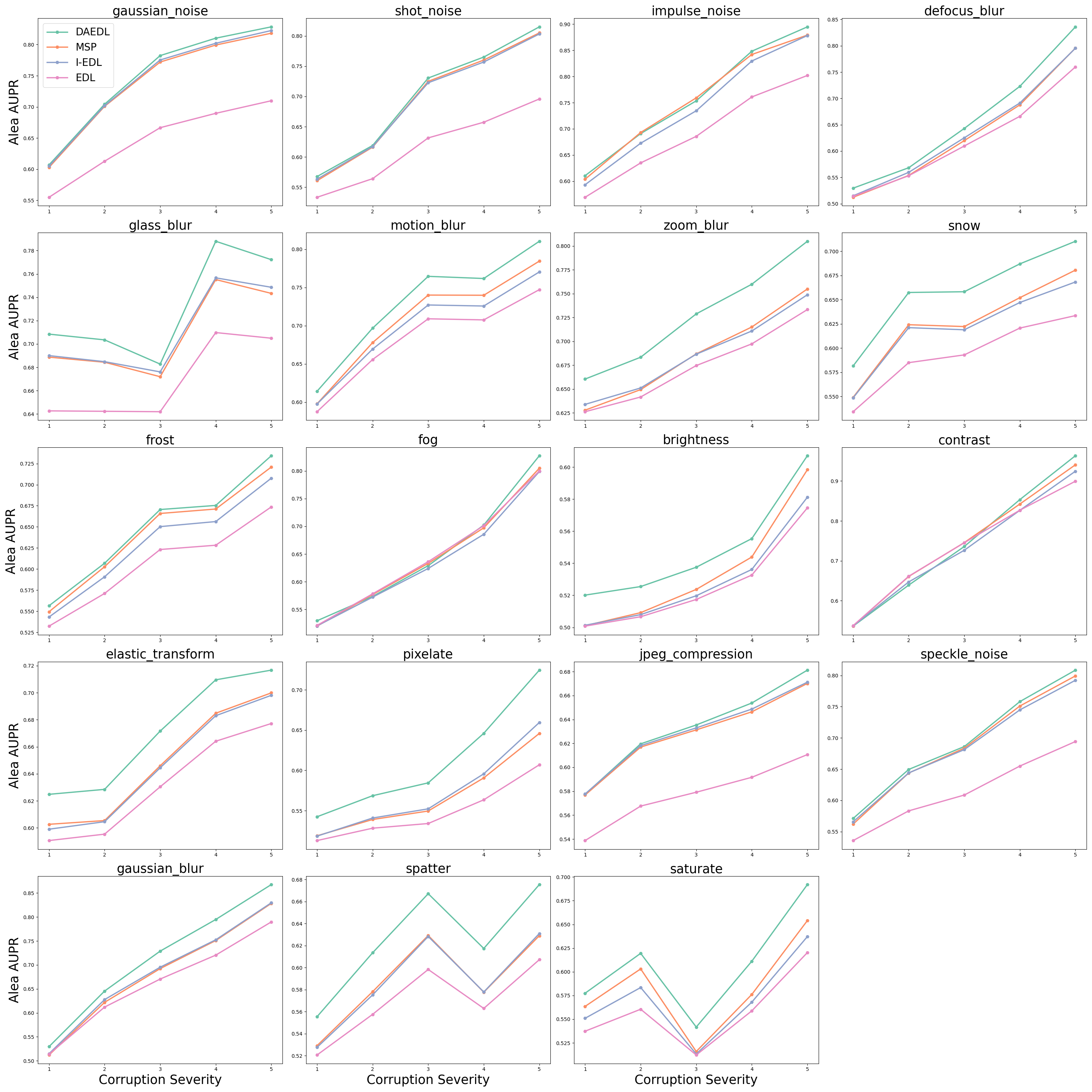}
\caption{AUPR scores for distribution shift detection using aleatoric uncertainty estimates across 19 different corruptions in the CIFAR-10-C dataset}
\label{Alea_AUPR_COMBINED}
\end{figure}

\noindent
\begin{figure}[hbtp]
\includegraphics[scale = 0.22]{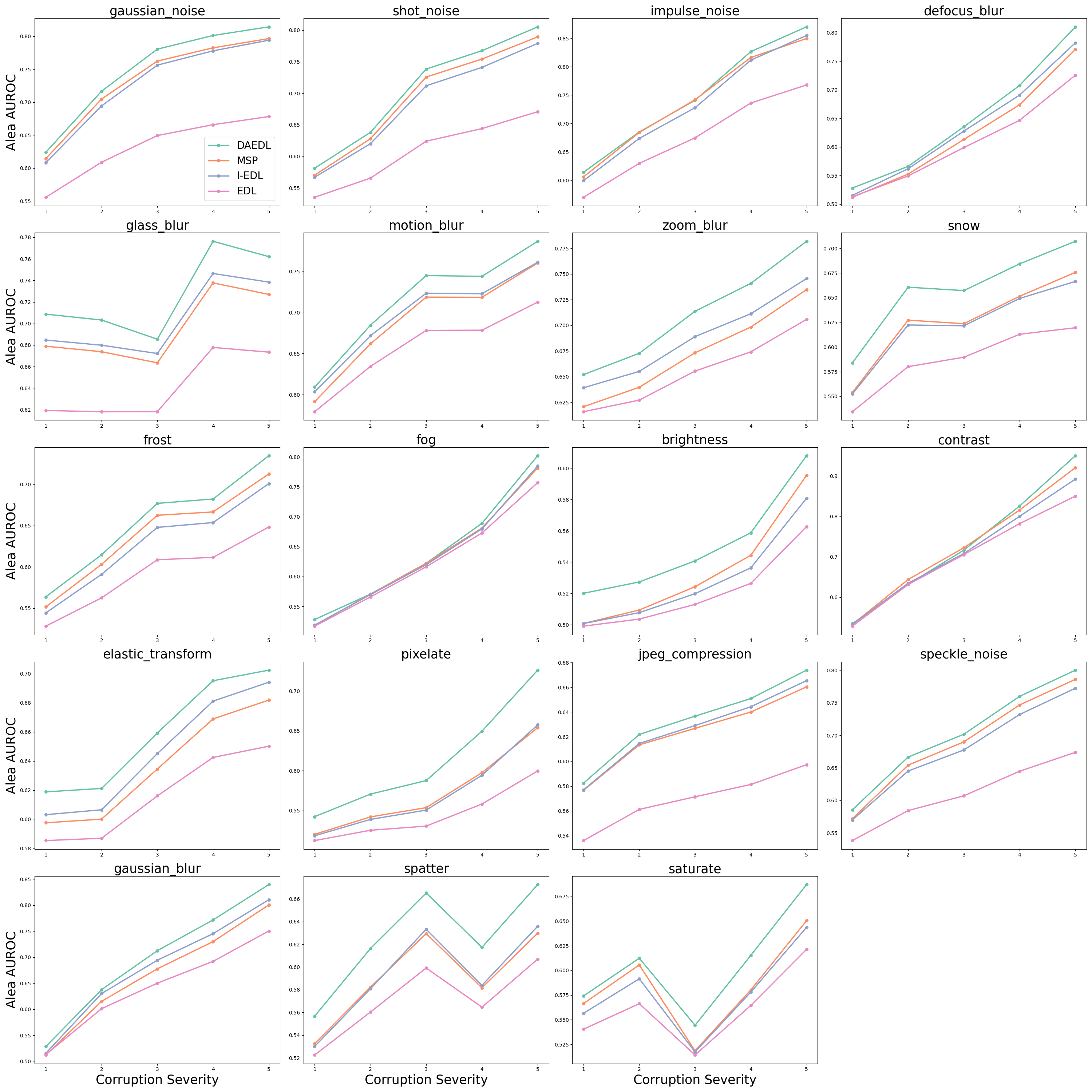}
\caption{AUROC scores for distribution shift detection using aleatoric uncertainty estimates across 19 different corruptions in the CIFAR-10-C dataset}
\label{Alea_AUROC_COMBINED}
\end{figure}

\noindent
\begin{figure}[hbtp]
\includegraphics[scale = 0.22]{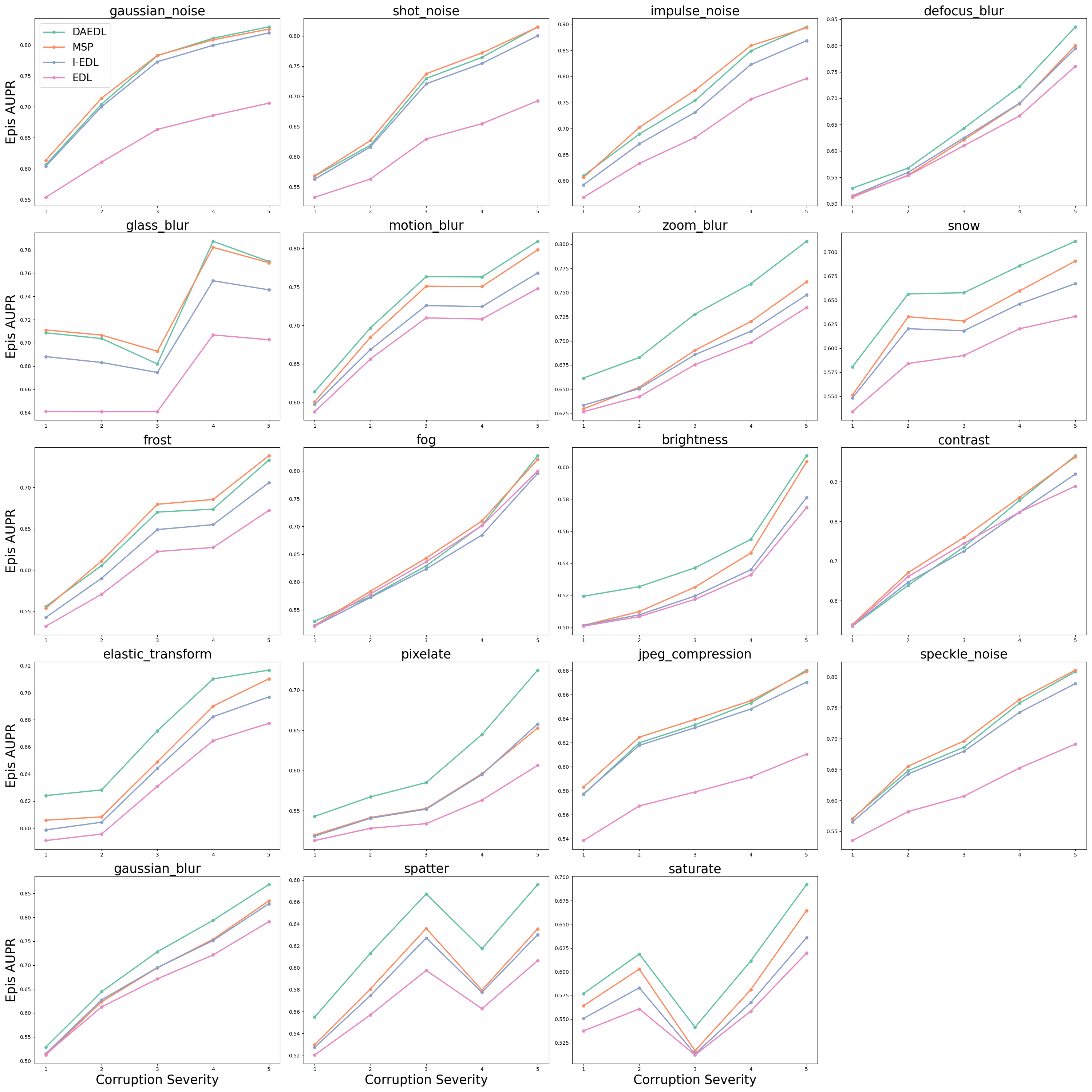}
\caption{AUPR scores for distribution shift detection using epistemic uncertainty estimates across 19 different corruptions in CIFAR-10-C dataset.}
\label{Epis_AUPR_Combined}
\end{figure}

\noindent
\begin{figure}[hbtp]
\includegraphics[scale = 0.22]{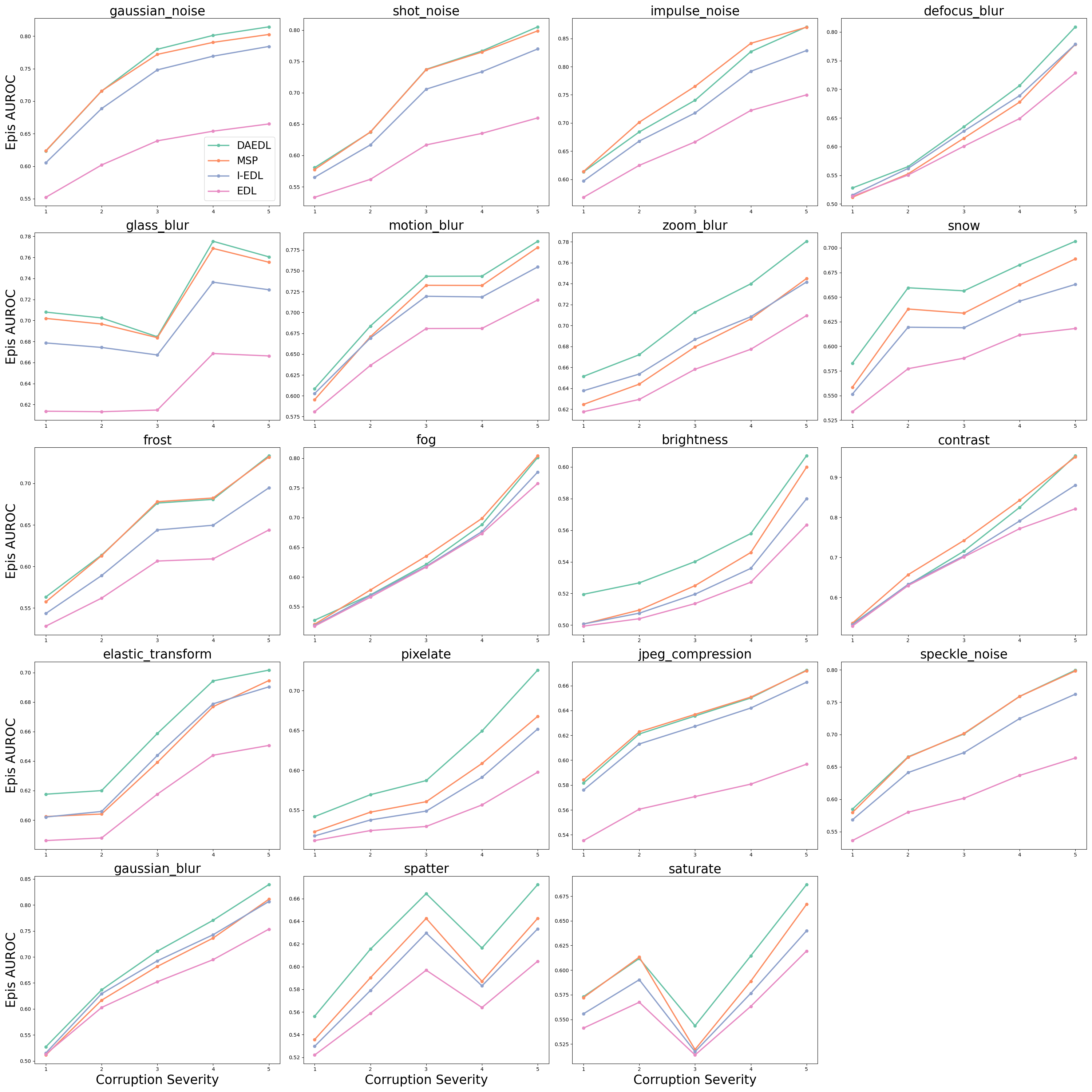}
\caption{AUROC scores for distribution shift detection using epistemic uncertainty estimates across 19 different corruptions in the CIFAR-10-C dataset}
\label{Epis_AUROC_Combined}
\end{figure}

\newpage
\subsection{Additional Results in Ablation Study (\cref{Exp : Ablation Study})}
\label{Appendix : Ablation Study}

\cref{Ablation Study - MNIST} presents the results of the ablation study conducted on the MNIST dataset. We can observe that each component of our model is effective individually and collectively demonstrates a synergistic effect, thereby enhancing the performance of EDL.

\begin{table*}[hbtp]
\caption{Ablation study results on MNIST. 
The results of EDL (DAEDL without EXP, DE, and SN) were from  \citet{deng2023uncertainty}}
\label{Ablation Study - MNIST}
\vskip 0.15in
\begin{center}
\begin{small}
\begin{sc}
\begin{tabular}{@{}lcc|ccccccc@{}}
\toprule
 &  &  & \multicolumn{2}{c}{\textbf{MNIST $\rightarrow$ KMNIST}} & \multicolumn{2}{c}{\textbf{MNIST $\rightarrow$ FMNIST}} \\ \midrule
\textbf{Exp} & \textbf{DE} & \textbf{SN} & \textbf{Alea.} & \textbf{Epis.} & \textbf{Alea.} & \textbf{Epis.} \\ \midrule
 \xmark & \xmark & \xmark & 97.02 $\pm$ 0.8 & 96.34  $\pm$ 2.0 & 98.10  $\pm$ 0.4 & 98.08  $\pm$ 0.4 \\ \midrule
\cmark & \xmark & \xmark & 98.86 $\pm$ 0.0 & 98.89 $\pm$ 0.0 & 99.35 $\pm$ 0.0 & 99.48 $\pm$ 0.0 \\
\cmark & \cmark & \xmark & 99.74 $\pm$ 0.0 & 99.76 $\pm$ 0.0 & 99.65 $\pm$ 0.1 & 99.67 $\pm$ 0.1 \\ 
 \cmark & \xmark & \cmark &  98.86 $\pm$ 0.1 & 98.90 $\pm$ 0.1  & 99.40 $\pm$ 0.1 & 99.50 $\pm$ 0.1 \\ \midrule
\cmark & \cmark & \cmark &   \textbf{99.90 $\pm$ 0.0} & \textbf{99.92 $\pm$ 0.0} & \textbf{99.83 $\pm$ 0.0} & \textbf{99.87 $\pm$ 0.0} \\ \bottomrule
\end{tabular} \\
\end{sc}
\end{small}
\end{center}
\vskip -0.15in
\end{table*}

\section{Additional Explanations about the Concepts}
To enhance the self-contained nature of our paper, we have included supplementary explanations for the concepts utilized in our study.
\paragraph{Dirichlet Distribution.}
The probability density function (PDF) of a Dirichlet distribution is formulated as follows:
\begin{align*}
\Dir(\boldsymbol{\pi}|\boldsymbol{\alpha}) = \frac{\Gamma(\sum_{c=1}^{C} \alpha_{c})}{\prod_{c=1}^{C} \Gamma(\alpha_{c})} \prod_{c=1}^{C} \pi_{c}^{\alpha_{c-1}},
\end{align*}
where $\boldsymbol{\pi} \in \Delta^{C-1}$ and $\boldsymbol{\alpha} = [\alpha_{1}, \alpha_{2}, \cdots, \alpha_{C}], \forall \alpha_{c} > 0$ being the concentration parameters. Here, the expected probability for the $c$th class can be calculated as follows:
\begin{align*}
\bar{\pi}_{c} &= \mathbb{E}_{\boldsymbol{\pi} \sim \Dir(\boldsymbol{\alpha)}}[\pi_{c}] = \frac{\alpha_{c}}{\alpha_{0}}.
\end{align*}
In this paper, we utilize $\max_{c} \bar{\pi}_{c}$ (i.e., maximum expected class probability) and $\alpha_{0}$ (i.e., precision of the Dirichlet distribution) as the measures of aleatoric and epistemic uncertainty, respectively. In addition to these metrics, DBU models provide various types of uncertainty measures in a closed form. For more comprehensive analysis and derivation of such measures, please refer to \citet{ulmer2023prior}. 

\paragraph{Spectral Normalization.}
\label{Appendix : Spectral Normalization}
Spectral normalization \cite{miyato2018spectral} is a technique that is applied to ensure a regularized feature space. It has been widely utilized in the context of single forward pass uncertainty estimation models. 
Specifically, spectral normalization is applied by estimating the spectral norm of the weight matrices and dividing the weight matrices by their spectral norm. 
Suppose that $f_{\boldsymbol{\theta}}$ is an $L$-layer neural network with weight matrices $\boldsymbol{\theta} = \{W^{(l)}\}_{l=1}^{L}$. We estimate the spectral norm $\sigma(W^{(l)})$ of the weight matrices for each layer using the power iteration method \cite{miyato2018spectral, gouk2021regularisation} and normalize the weights by dividing it by the corresponding spectral norm as follows:
\begin{align*}
\quad W^{(l)} \leftarrow \frac{W^{(l)}}{\sigma(W^{(l)})}, \quad \sigma(W^{(l)}) = \underset{h : h \neq 0}{\max} \frac{\lVert W^{(l)}h\rVert_{2}}{\lVert h \rVert_{2}}.
\end{align*}
For the updated $f_{\boldsymbol{\theta}}$ with the spectral normalized weight, the Lipschitz norm $\lVert f_{\boldsymbol{\theta}} \rVert_{\Lip} $ is bounded above by 1. Thus, we can ensure that the feature space distance is bounded by the input space distance. In other words, inequality $\lVert f_{\boldsymbol{\theta}}(\boldsymbol{x}_{1}) - f_{\boldsymbol{\theta}}(\boldsymbol{x}_{2}) \rVert_{2} \leq \lVert \boldsymbol{x}_{1} - \boldsymbol{x}_{2} \rVert_{2}$ holds $\forall \boldsymbol{x}_{1}, \boldsymbol{x}_{2} \in \mathbb{R}^{D}$. This prevents the feature representations from being overly sensitive to the meaningless perturbation in the input space and ensures the representation to be more informative \cite{liu2020simple}. 

In this paper, we utilize spectral normalization to obtain a regularized feature space that enables meaningful feature space density estimates with GDA. 

\paragraph{Distance Awareness.}
\label{Appendix : distance awareness}
Distance awareness is a beneficial property that ensures the model to obtain high-quality uncertainty estimates. It has been widely utilized in the context of single forward pass uncertainty estimation models. Formally, distance awareness is defined as \cref{distance awareness}

\begin{definition} 
\label{distance awareness}
\textbf{(Distance Awareness) \cite{liu2020simple}} The predictive distribution $p(y|\mathbf{x})$ is \textit{distance aware} if there exist $u(\mathbf{x})$, a summary statistic of $p(y|\mathbf{x})$ that quantifies model uncertainty that reflects distance between $\mathbf{x}$ and the training data with respect to $\lVert \rVert_{\mathcal{X}}$ i.e., $u(\mathbf{x}) = v(d(\mathbf{x},\mathcal{X}))$, where $v$ is a a monotonic function and $d(\mathbf{x},\mathcal{X}) = \mathbb{E}_{\mathbf{x}' \sim \mathcal{X}}\lVert \mathbf{x}-\mathbf{x}' \rVert_{\mathcal{X}}^{2}$ is the distance between $\mathbf{x}$ and the training data domain. 
\end{definition} 
In this paper, we prove that DAEDL satisfies \textit{distance awareness} with respect to both the feature space and input space under certain conditions (\cref{Feature Distance Awareness}, \cref{Input Distance Awareness}). 
\end{document}